%% file: TIT_final_v4.tex
\documentclass[draftcls,onecolumn]{IEEEtran}
\input{preamble}

\usepackage[ colorlinks = true,
             linkcolor = blue,
             urlcolor  = blue,
             citecolor = red,
             anchorcolor = green,
]{hyperref}

\usepackage{cite}

\usepackage{amssymb}
\usepackage{amsthm}
\usepackage{amsmath}
\usepackage{mathrsfs}
\usepackage{amsmath,graphicx,bm,xcolor,url}
\usepackage{subfig}
\usepackage{xr}

\begin{document}
    
\title{The Informativeness of $K$-Means for Learning Mixture Models}

\author{Zhaoqiang Liu, Vincent Y.~F.\ Tan

\thanks{
Z.~Liu is with the Department of Computer Science, National University of Singapore (email: \url{dcslizha@nus.edu.sg}). 

V.~Y.~F.\ Tan is with the Department of Electrical and Computer Engineering and the Department of Mathematics, National University of Singapore (email: \url{vtan@nus.edu.sg}).

This work was supported in part by an NRF Fellowship (R-263-000-D02-281) and in part by a Ministry of Education Tier 2 grant (R-263-000-C83-112).}}

\maketitle

\begin{abstract}
    The learning of mixture models can be viewed as a clustering problem. Indeed, given data samples independently generated from a mixture of distributions, we often would like to find the {\it correct target clustering} of the samples according to which component distribution they were generated from. For a clustering problem, practitioners often choose to use the simple $k$-means algorithm. $k$-means attempts to find an {\it optimal clustering} that minimizes the sum-of-squares distance between each point and its cluster center. In this paper, we consider fundamental (i.e., information-theoretic) limits of the solutions (clusterings) obtained by optimizing the sum-of-squares distance. In particular, we provide sufficient conditions for the closeness of any optimal clustering and the correct target clustering assuming that the data samples are generated from a mixture of spherical Gaussian distributions. We also generalize our results to log-concave distributions. Moreover, we show that under similar or even weaker conditions on the mixture model, any optimal clustering for the samples with reduced dimensionality is also close to the correct target clustering. These results provide intuition for the informativeness of $k$-means (with and without dimensionality reduction) as an algorithm for learning mixture models.
\end{abstract}

\section{Introduction}\label{sec:intro}

Suppose there are $K$ unknown distributions $F_1,F_2,\ldots,F_K$   and a probability vector $\bw:=[w_1,w_2,\ldots,w_K]$. The corresponding $K$-component mixture model is a generative model that assumes data samples are independently sampled  such that the probability that each sample is generated from the $k$-th component is $w_k$, the {\em mixing weight for the $k$-th component}. Suppose that $\bv_1,\bv_2,\ldots,\bv_N$ are samples independently generated from a $K$-component mixture model, the {\em correct target clustering} $\mathscr{C}:=\{\mathscr{C}_1,\mathscr{C}_2,\ldots,\mathscr{C}_K\}$ satisfies the condition that $n\in \mathscr{C}_k$ if and only if $\bv_n$ is generated from the $k$-th component. One of the most important goals of the learning of a mixture model is to find the correct target clustering of the samples (and thereby inferring the parameters of the model). The learning of mixture models, especially Gaussian mixture models (GMMs)~\cite{titterington1985, bishop2006}, is of fundamental importance.

Clustering is a ubiquitous problem in various applications, such as analyzing  the information contained in gene expression data and performing market research according to firms' financial characteristics. Objective-based clustering is a commonly-used technique for clustering. This is the procedure of  minimizing a certain objective function to partition data samples into a fixed or appropriately-selected number of subsets known as {\em clusters}. The $k$-means algorithm~\cite{lloyd1982} is perhaps the most popular objective-based clustering approach. Suppose we have a data matrix of $N$ samples $\bV=[\bv_1,\bv_2,\ldots,\bv_N] \in \mathbb{R}^{F\times N}$, a {\em $K$-clustering} (or simply a {\em clustering} or a {\em partition}) is defined as a set of pairwise disjoint index  sets $\mathscr{C}:=\{\mathscr{C}_1,\mathscr{C}_2,\ldots,\mathscr{C}_K\}$ whose union  is $\{1,2,\ldots,N\}$. The corresponding {\em sum-of-squares distortion measure} with respect to $\bV$ and $\mathscr{C}$ is defined as 
\begin{equation}
 \mathcal{D}(\bV, \mathscr{C}) := \sum_{k=1}^{K} \sum_{n \in \mathscr{C}_k} \|\bv_n - \bc_k\|_2^2, \label{eqn:sos}
\end{equation}
where $\bc_k:=\frac{1}{|\mathscr{C}_k|}\sum_{n \in \mathscr{C}_k} \bv_n$ is the cluster center or centroid of the $k$-th cluster. The goal of  the $k$-means algorithm is to find an {\em optimal clustering} $\mathscr{C}^{\mathrm{opt}}$ that satisfies
\begin{equation}
 \mathcal{D}(\bV, \mathscr{C}^{\mathrm{opt}}) = \min_{\mathscr{C}} \mathcal{D}(\bV, \mathscr{C}),\label{eqn:sos_min}
\end{equation}
where the minimization is taken over all $K$-clusterings. Optimizing this objective function is NP-hard \cite{dasgupta2008}. Despite the wide usage of $k$-means and some analysis of the $k$-means algorithm~\cite{arthur2007, bottou1995, blomer2016}, there are relatively fewer theoretical investigations of the {\em fundamental} or {\em information-theoretic performance limits} of {\em optimal clusterings} obtained via optimizing~\eqref{eqn:sos}. In real applications, such as clustering face images by identities, there are certain unknown correct target clusterings. While using $k$-means as a clustering algorithm, we make a {\it key implicit assumption} that any optimal clustering is close to the correct target clustering \cite{balcan2009}. If there is an optimal clustering that is far away from the correct target clustering, then using $k$-means  is meaningless because even if we obtain an optimal or approximately-optimal clustering, it may not be close to the desired  correct target clustering. 

Furthermore, due to the inherent inefficiencies in processing high-dimensional data, {\em dimensionality reduction} has received considerable attention. Applying such techniques before clustering high-dimensional datasets can lead to significantly faster running times  and reduced memory sizes. In addition, algorithms for learning GMMs  usually include a dimensionality reduction step. For example, Dasgupta \cite{dasgupta1999} shows that general ellipsoidal Gaussians become ``more spherical'' and thereby more amenable to (successful) analysis after a random projection onto a low-dimensional subspace. Vempala and Wang \cite{vempala2002} show that reducing dimensionality by spectral decomposition leads to the amplification of the separation among Gaussian components. 

\subsection{Main Contributions}
There are three main  contributions in this paper, all concerned with fundamental limits of optimal solutions to~\eqref{eqn:sos_min}.
\begin{enumerate}
 \item We prove that if the data points are independently generated from a $K$-component spherical GMM with an appropriate separability assumption and the so-called non-degeneracy condition \cite{hsu2013,anandkumar2014} (see Definition~\ref{def: non-degen} to follow), then any optimal clustering of the data points is close to the correct target clustering with high probability provided the number of samples is commensurately large. We extend this result to mixtures of log-concave distributions.
 
 \item  We prove that under the same data generation process, if the data points  are projected onto a low-dimensional space
using the first $K - 1$ principal components of
the empirical covariance matrix, then, under similar conditions, any optimal clustering for the data points
with reduced dimensionality is close to the correct
target clustering with high probability. Again, this result is extended to mixtures of log-concave distributions.

 \item We show that under appropriate conditions, any {\it approximately-optimal clustering} of the data points is close to the correct target clustering. This enables us to use the theoretical framework provided herein to analyze various dimensionality reduction techniques  such as random projection and randomized singular value decomposition (SVD). It also allows us to combine our theoretical analyses with efficient variants of $k$-means that   return  approximately-optimal clusterings.
\end{enumerate}

\subsection{Notations} \label{sec: notations}
We use upper and lower case boldface letters to denote matrices and vectors respectively. We use $\mathrm{diag}(\bw)$ to represent the diagonal matrix whose diagonal entries are given by $\bw$. The Frobenius and spectral norms of $\bV$ are written as  $\|\bV\|_{\mathrm{F}}$  and $\|\bV\|_2$ respectively.  We use $\mathrm{tr}(\bX)$ to represent the trace of matrix $\bX$. The matrix inner product between $\bA$ and $\bB$ is written as  $\langle\bA,\bB\rangle:=\mathrm{tr}(\bA^T\bB)$. We write $[N]=\{1,2,\cdots,N\}$. Let $\bV_{1} \in \mathbb{R}^{F\times N_{1}}$ and $\bV_{2} \in \mathbb{R}^{F\times N_{2}}$, we denote by $\left[\bV_{1}, \bV_{2}\right]$ the horizontal concatenation of the two matrices. We use $\bV(1\colon I,\colon)$ to denote the first $I$ rows of $\bV$ and use $\bV(\colon,1\colon J)$ to denote the first $J$ columns of $\bV$;  in particular, $\bV(i,j)$ is the element in the $(i,j)^{\mathrm{th}}$  position of $\bV$. The SVD of a {\it symmetric} matrix $\bA \in \mathbb{R}^{F\times F}$ is given by $\bA=\bU\bD\bU^T$ with $\bU \in \mathbb{R}^{F \times F}$ being an orthogonal matrix and $\bD \in \mathbb{R}^{F \times F}$ being a diagonal matrix. In addition, when $R:=\mathrm{rank}(\bA) < F$, the so-called compact SVD of $\bA$ is $\bA=\bU_R\bD_R\bU_R^T$, where $\bU_R:=\bU(\colon,1\colon R)$ and $\bD_R:=\mathbf{\Sigma}(1\colon R,1\colon R)$. The {\em centering} of any data matrix $\bV=[\bv_1,\bv_2,\ldots,\bv_N]$ is a shift with respect to the mean vector  $\bar{\bv}=\frac{1}{N}\sum_{n=1}^{N}\bv_n$ and the resultant data matrix $\bZ=[\bz_1,\bz_2,\ldots,\bz_N]$, with $\bz_n=\bv_n-\bar{\bv}$ for $n \in [N]$, is said to be the {\em centralized matrix} of $\bV$. For a $K$-component mixture model and for any $k\in[K]$, we {\em always} use $\bu_k$ to denote the component mean vector, and use $\mathbf{\Sigma}_k$ to denote the component covariance matrix. When $\mathbf{\Sigma}_k=\sigma_k^2\bI$ for all $k\in [K]$, where $\bI$ is the identity matrix, we say the mixture model is {\em spherical} and $\sigma_k^2$ is the variance of the $k$-th component. 

\subsection{Paper Outline}\label{sec: outline}
This paper is organized as follows. In Section \ref{sec: rel_work}, we mention some related works on learning mixture models, $k$-means clustering, and dimensionality reduction. Our main theoretical results concerning upper bounds on the misclassification error distance for spherical GMMs are presented in Section \ref{sec: main_thm}. These results are extended to mixtures of log-concave distributions in Section~\ref{sec: main_thm_log}. Other extensions  using other dimensionality-reduction techniques and the combination of our results with efficient clustering algorithms are discussed in Section~\ref{sec:other}. We conclude our discussion and suggest avenues for future research in Section~\ref{sec:concl}. Proofs are relegated to the various appendices.

\section{Related Work}\label{sec: rel_work}  
In this section, we discuss some relevant existing works.

\subsection{Learning (Gaussian) Mixture Models} \label{sec:gmm} 
The learning of GMMs is of fundamental importance in machine learning and its two most important goals are: (i) Inferring the parameters of the GMM; (ii) Finding the correct target clustering of data samples according to which Gaussian distribution they were generated from. The EM algorithm \cite{dempster1977, mclachlan2007} is widely used to estimate the parameters of a GMM. However, EM is a local-search heuristic approach for maximum likelihood estimation in the presence of incomplete data and in general, it cannot guarantee the parameters' convergence to global optima \cite{wu1983}. Recently, Hsu and Kakade \cite{hsu2013} and Anandkumar {\em et al.}~\cite{anandkumar2014} provided approaches based on spectral decomposition to obtain consistent parameter estimates for spherical GMMs  from first-, second- and third-order observable moments. To estimate parameters, they need to assume the so-called non-degeneracy condition for {\em spherical} GMMs with parameters $\{ (w_k, \bu_k,\sigma_k^2)\}_{k\in [K]}$.
\begin{definition}(Non-degeneracy condition)\label{def: non-degen}
We say that a mixture model satisfies the {\em non-degeneracy condition} if its component mean vectors $\bu_1,\ldots,\bu_K$ span a $K$-dimensional subspace and the probability vector $\bw$ has strictly positive entries.
\end{definition}
On the other hand, under certain separability assumptions on the Gaussian components, Dasgupta~\cite{dasgupta1999}, Dasgupta and Schulman~\cite{dasgupta2000}, Arora and Kannan \cite{arora2001}, Vempala and Wang \cite{vempala2002} and Kalai {\em et al.}~\cite{kalai2010}  provided provably correct algorithms that guarantee that most samples are correctly classified or that parameters are recovered with a certain accuracy with high probability. In particular, equipped with the following {\em separability} assumption 
\begin{align} 
 \|\bu_i - \bu_j\|_2  &>  C \max\{\sigma_i, \sigma_j\} \sqrt[4]{K\log  \frac{F}{w_{\min}}},  \quad \forall, i, j \in [K],  i \neq j,\label{eq: well-sep}
\end{align}
 for a spherical GMM, where $C>0$ is a sufficiently large constant\footnote{Throughout, we use the generic notations $C$ and $C_i,i\in\mathbb{N}$ to denote (large) positive constants. They depend on the parameters of the mixture model and may change from usage to usage.} and $w_{\min}:= \min_{k \in [K]} w_k$, Vempala and Wang \cite{vempala2002} present a simple spectral algorithm with running time polynomial in both $F$ and $K$ that correctly clusters  random samples according to which spherical Gaussian they were generated from.
 
 Despite the large number of algorithms designed to find the (approximately) correct target clustering of a GMM, many practitioners use $k$-means because of its simplicity and successful applications in various fields. Kumar and Kannan \cite{kumar2010} show that the $k$-means algorithm with a proper initialization can correctly cluster nearly all the data points generated from a GMM that satisfies a certain {\em proximity} assumption. Our theoretical results provide an explanation on {\em why} the $k$-means algorithm that attempts to find an optimal clustering is a good choice for learning mixture models. We compare and contrast our work to that of~Kumar and Kannan \cite{kumar2010} in Remark~\ref{rmk:KK}. 
 
 \subsection{Lower Bound on Distortion and  the ME Distance} \label{sec: imp_lemmas}
Let  $\bV=[\bv_1,\bv_2,\ldots,\bv_N] \in \mathbb{R}^{F\times N}$ be a dataset and $\mathscr{C}:=\{\mathscr{C}_1,\mathscr{C}_2,\ldots,\mathscr{C}_K\}$ be a $K$-clustering. Let $\bH \in \mathbb{R}^{K\times N}$ with elements $\bH(k,n)$ be the clustering membership matrix satisfying $\bH(k,n)=1$ if $n \in \mathscr{C}_k$ and $\bH(k,n)=0$ if $n \notin \mathscr{C}_k$  for $(k,n)\in[K]\times[N]$. Let $n_k=|\mathscr{C}_k|$  and $\bar{\bH} := \mathrm{diag}(\frac{1}{\sqrt{n_1}},\frac{1}{\sqrt{n_2}},\ldots,\frac{1}{\sqrt{n_K}}) \bH$ be the normalized version of $\bH$. We have $\bar{\bH}\bar{\bH}^T=\bI$ and the corresponding distortion can be written as \cite{ding2004}
\begin{equation} \label{eq: distortion_basis}
 \mathcal{D}(\bV, \mathscr{C})=\|\bV-\bV\bar{\bH}^T\bar{\bH}\|_{\mathrm{F}}^2 = \|\bV\|_{\mathrm{F}}^2-\mathrm{tr}(\bar{\bH}\bV^T\bV\bar{\bH}^T).
\end{equation}
Let $\bZ$ be the centralized data matrix of $\bV$ and define $\bS:=\bZ^T \bZ$. Note that $\mathcal{D}(\bV, \mathscr{C}) = \mathcal{D}(\bZ, \mathscr{C})$ for any clustering $\mathscr{C}$. Ding and He \cite{ding2004} make use of this property to provide a lower bound $\mathcal{D}^*(\bV)$ for distortion over all $K$-clusterings. That is, for any $K$-clustering $\mathscr{C}$, 
\begin{equation}
 \mathcal{D}(\bV, \mathscr{C}) \ge \mathcal{D}^*(\bV) := \mathrm{tr}(\bS)-\sum_{k=1}^{K-1} \lambda_k(\bS),
\end{equation}
where $\lambda_1(\bS)\ge \lambda_2(\bS) \ge \ldots \ge 0$ are the   eigenvalues of $\bS$ sorted in non-increasing order.

For any two $K$-clusterings, the so-called \textit{misclassification error (ME) distance} provides a quantitative comparison of  their structures.
\begin{definition} (ME distance)
 The misclassification error distance of any two $K$-clusterings $\mathscr{C}^{(1)}:=\{\mathscr{C}_1^{(1)},\mathscr{C}_2^{(1)},\ldots,\mathscr{C}_K^{(1)}\}$ and $\mathscr{C}^{(2)}:=\{\mathscr{C}_1^{(2)},\mathscr{C}_2^{(2)},\ldots,\mathscr{C}_K^{(2)}\}$ is 
 \begin{equation}
  \mathrm{d}_{\mathrm{ME}}(\mathscr{C}^{(1)},\mathscr{C}^{(2)}):= 1- \frac{1}{N}\max_{\pi \in \mathcal{P}_K} \sum_{k=1}^{K} \left|\mathscr{C}_k^{(1)} \cap \mathscr{C}_{\pi(k)}^{(2)} \right|,
 \end{equation}
 where $\mathcal{P}_K$ is the set of all permutations of $[K]$. It is known from Meil{\u{a}} \cite{meilua2005} that the ME distance is indeed a distance.
\end{definition}

For any $\delta, \delta' \in [0,K-1]$, define the functions
\begin{align}
 \tau(\delta,\delta') &:= 2\sqrt{\delta\delta'\Big(1-\frac{\delta}{K-1}\Big)\Big(1- \frac{\delta'}{ K-1}\Big)},\quad\mbox{and}\label{eq: tau_2}\\
 \tau(\delta) &:= \tau(\delta,\delta) =  2 \delta\Big(1- \frac{\delta}{ K-1 }  \Big).\label{eq: tau}
\end{align}
Combining Lemma 2 and Theorem 3 in Meil{\u{a}} \cite{meilua2006}, we have the following lemma. 
\begin{lemma}\label{lem: ME_dist}
Let $\mathscr{C}:=\{\mathscr{C}_1,\mathscr{C}_2,\ldots,\mathscr{C}_K\}$ and $\mathscr{C}':=\{\mathscr{C}'_1,\mathscr{C}'_2,\ldots,\mathscr{C}'_K\}$ be two $K$-clusterings of a dataset $\bV \in \mathbb{R}^{F\times N}$. Let  $p_{\max}:=\max_k \frac{1}{N}|\mathscr{C}_k|$ and $p_{\min}:=\min_k \frac{1}{N} |\mathscr{C}_k|$. Let $\bZ$ be the centralized matrix of $\bV$ and $\bS=\bZ^T \bZ$. Define
\begin{equation} \label{eq: delta}
  \delta:=\frac{\mathcal{D}(\bV, \mathscr{C})-\mathcal{D}^*(\bV)}{\lambda_{K-1}(\bS)-\lambda_{K}(\bS)}, \quad\mbox{and}\quad \delta':=\frac{\mathcal{D}(\bV, \mathscr{C}')-\mathcal{D}^*(\bV)}{\lambda_{K-1}(\bS)-\lambda_{K}(\bS)}.
 \end{equation}
 Then if $\delta, \delta' \le \frac{1}{2}(K-1)$ and $\tau(\delta,\delta') \le p_{\min}$, we have 
 \begin{equation}
  \mathrm{d}_{\mathrm{ME}}(\mathscr{C}, \mathscr{C}') \le \tau(\delta,\delta') p_{\max}.
 \end{equation} 
\end{lemma}
This lemma says that any two ``good'' $K$-clusterings (in the sense that their distortions are sufficiently close to the lower bound of distortion ${\cal D}^*(\bV)$) are close to each other. In addition, we have the following useful corollary.

\begin{corollary} \label{coro: main}
 Let $\mathscr{C}:=\{\mathscr{C}_1,\mathscr{C}_2,\ldots,\mathscr{C}_K\}$ be a $K$-clustering of a dataset $\bV \in \mathbb{R}^{F\times N}$ and define $p_{\max}$, $p_{\min}$, $\bZ$,   $\bS$, and $\delta$ as in Lemma \ref{lem: ME_dist}. Then if $\delta \le \frac{1}{2}(K-1)$ and $\tau(\delta) \le p_{\min}$, we have
\begin{equation}
 \mathrm{d}_{\mathrm{ME}}(\mathscr{C},\mathscr{C}^{\mathrm{opt}}) \le p_{\max} \tau(\delta),
\end{equation}
where $\mathscr{C}^{\mathrm{opt}}$ represents a $K$-clustering that minimizes the distortion for $\bV$.
\end{corollary}
This corollary essentially says that if the distortion of a clustering is sufficiently close to the lower bound of distortion, then this clustering is close to any optimal clustering with respect to the ME distance. 

\subsection{Principal Component Analysis}\label{sec: pca}
Principal component analysis (PCA) \cite{jolliffe1986} is a  popular strategy to compute the directions of maximal variances in vector-valued data and is widely used for dimensionality reduction. For any dataset $\bV \in \mathbb{R}^{F\times N}$ and any positive integer $k \le F$, the so-called $k$-PCA for the dataset usually consists of two steps: (i) Obtain the centralized dataset $\bZ$; (ii) Calculate the SVD of $\bar{\mathbf{\Sigma}}_N:=\frac{1}{N}\bZ\bZ^T$, i.e., obtain $\bar{\mathbf{\Sigma}}_N=\bP\bD\bP^T$, and project the dataset onto a $k$-dimensional space to obtain $\tilde{\bV}:=\bP_k^T\bV$, where $\bP_k:=\bP(\colon,1\colon k)$. For brevity, we say that $\tilde{\bV}$ is the {\em post-$k$-PCA dataset} of $\bV$ (or simply the {\em post-PCA} dataset). If   only the projection step is performed (and not the centralizing step), we  term the corresponding approach {\em $k$-PCA  with no centering} or simply {\em  $k$-SVD}, and we say that the corresponding $\tilde{\bV}$ is the {\em post-$k$-SVD dataset} (or simply the {\em post-SVD} dataset) of $\bV$.

\subsection{Comparing Optimal Clusterings for the Original Dataset and the Post-PCA Dataset}
When performing dimensionality reduction for clustering, it is important to compare any optimal clustering for the dataset with reduced dimensionality to any optimal clustering for the original dataset. More specifically, any optimal clustering for the dataset with reduced dimensionality should be close to any optimal clustering for the original dataset. However, existing works  \cite{boutsidis2015, cohen2015} that combine  $k$-means clustering and dimensionality reduction can only guarantee that the distortion of any optimal clustering for the dataset with reduced dimensionality, $\tilde{\mathscr{C}}^{\mathrm{opt}}$, can be bounded by a factor $\gamma >1 $ times the distortion of any optimal clustering for the original dataset,
$\mathscr{C}^{\mathrm{opt}}$. That is, 
\begin{equation} \label{eq: distortion_bd}
 \mathcal{D}(\bV,\tilde{\mathscr{C}}^{\mathrm{opt}}) \le \gamma \mathcal{D}(\bV,\mathscr{C}^{\mathrm{opt}}).
\end{equation}
As mentioned in Boutsidis {\em et al.} \cite{boutsidis2015}, directly comparing the {\em structures} of $\tilde{\mathscr{C}}^{\mathrm{opt}}$ and $\mathscr{C}^{\mathrm{opt}}$ (instead of their distortions) is more interesting. In this paper, we also prove that, if  the samples are generated from a spherical GMM (or a mixture of log-concave distributions)  which satisfies a separability assumption and the non-degeneracy condition, when the number of samples is sufficiently large, the ME distance between any optimal clustering for the original dataset and any optimal clustering for the post-PCA dataset can be bounded appropriately.

In addition, we can show that any optimal clustering of the dimensionality-reduced dataset is close to the correct target clustering by leveraging~\eqref{eq: distortion_bd}. This simple strategy seems to be adequate for data-independent dimensionality reduction techniques such as random projection. However, for data-dependent dimensionality reduction techniques such as PCA, we believe that it is worth applying distinct proof techniques similar to those developed herein to obtain stronger theoretical results because of the generative models we assume. See Section~\ref{sec: other_dim_reduction} for a detailed discussion.

\section{Results for Spherical GMMs} \label{sec: main_thm} 
In this section, we assume the datasets are generated from spherical GMMs. Even though we can and will make statements for more general log-concave distributions (see Section~\ref{sec: main_thm_log}), this assumption allows us to illustrate our results and    mathematical ideas as cleanly as possible. We first present our main theorems for the upper bounds of ME distance between any optimal clustering and the correct target clustering for both the original and the dimensionality-reduced datasets in Section \ref{sec: des_main_thm}. We augment our results by several remarks and illustrative examples. In Section~\ref{sec: num_exp}, we conduct some numerical experiments to verify the correctness and tightness of the bounds presented in Section \ref{sec: des_main_thm}. In Section~\ref{sec: def_lem}, we provide several useful lemmas. Finally, in Sections \ref{sec: proof_main_thm} and~\ref{sec:prfthm2}, we provide our proofs (or proof sketches) for the theorems. 

\subsection{Description of the Theorems } \label{sec: des_main_thm} 
First, we show that with the combination of a new separability assumption and the non-degeneracy condition (cf.\ Definition~\ref{def: non-degen}) for a spherical GMM,  any optimal clustering for a dataset generated from the spherical GMM is close to the correct target clustering with high probability when the number of samples is sufficiently large. 

We   adopt the following set of notations. Let $\mathbf{\Sigma}_{N}:=\frac{1}{N}\bV\bV^T$ and $\bar{\mathbf{\Sigma}}_{N}:=\frac{1}{N}\bZ\bZ^T$, where $\bZ$ is the centralized matrix of $\bV$. Fix a mixture model with parameters $\{ (w_k,\bu_k,\mathbf{\Sigma}_k )\}_{k\in [K]}$ where $w_k$, $\bu_k$ and $\mathbf{\Sigma}_k$ denote the mixing weight, the mean vector, and the covariance matrix of the $k$-th component. Let 
\begin{align}
\mathbf{\Sigma}:=\sum_{k=1}^{K} w_k\left(\bu_k\bu_k^T +    \mathbf{\Sigma}_k \right),\quad\mbox{and}\quad
\mathbf{\Sigma}_0:=\sum_{k=1}^{K}w_k\bu_k\bu_k^T.
\end{align}
 Denote $\bar{\bu}:=\sum_{k=1}^{K}w_k\bu_k$ and write 
\begin{align}
 \bar{\mathbf{\Sigma}}&:=\sum_{k=1}^{K}w_k\left((\bu_k-\bar{\bu})(\bu_k-\bar{\bu})^T +   \mathbf{\Sigma}_k \right), \quad \bar{\mathbf{\Sigma}}_0 :=\sum_{k=1}^{K}w_k(\bu_k-\bar{\bu})(\bu_k-\bar{\bu})^T,
\end{align} 
 and $\lambda_{\min} := \lambda_{K-1}(\bar{\mathbf{\Sigma}}_0)$.
For a $K$-component spherical mixture model with covariance matrices $\sigma_k^2\bI$ for $k \in [K]$, we write $\bar{\sigma}^2:=\sum_{k=1}^{K}w_k\sigma_k^2$.
 
For $p \in [0,\frac{1}{2}(K-1)]$, we define the function 
\begin{equation}\label{eqn:def_zeta}
 \zeta(p):=\frac{p}{1+\sqrt{1- \frac{2p}{K-1}}}. 
\end{equation}
We clearly have $\frac{1}{2}p \le \zeta(p) \le p$.   Our first theorem reads:
\begin{theorem} \label{thm: original}
Suppose all the columns of data matrix $\bV \in \mathbb{R}^{F\times N}$ are independently generated from a $K$-component spherical GMM and $N>F>K$. Assume the spherical GMM satisfies the non-degeneracy condition (cf.~Definition~\ref{def: non-degen}). Let $w_{\min}:=\min_{k} w_k$ and $w_{\max}:=\max_{k} w_k$. Further   assume   that
 \begin{equation} \label{eq: main_condition}
  \delta_0:=\frac{(K-1)\bar{\sigma}^2}{\lambda_{\min}} < \zeta( w_{\min}).
 \end{equation}
 Let $\mathscr{C}:=\{\mathscr{C}_1,\mathscr{C}_2,\ldots,\mathscr{C}_K\}$  be the correct target $K$-clustering corresponding to the spherical GMM. Assume that  $\epsilon>0$ satisfies 
 \begin{align} 
\epsilon &\le \min \left\{\frac{w_{\min}}{2}, \lambda_{\min}, (K-1)\bar{\sigma}^2\right\}     \quad \mbox{and} \quad \frac{(K-1)\bar{\sigma}^2+\epsilon}{\lambda_{\min} - \epsilon} \le \zeta(w_{\min}-\epsilon). \label{eq: epsilon_cond}
\end{align}
Then for any  $t\ge 1$, if $N\ge CF^5 K^2 t^2/\epsilon^2$, where $C>0$ depends on $\{  (w_k,\bu_k,\sigma_k^2)\}_{k\in[ K]}$, we have, with probability at least $1-36KF^2\exp\left(-t^2 F\right)$, 
\begin{equation}\label{eq: res_origin}
 \mathrm{d}_{\mathrm{ME}}(\mathscr{C},\mathscr{C}^{\mathrm{opt}}) \le \tau\left(\frac{(K-1)\bar{\sigma}^2+\epsilon}{\lambda_{\min}-\epsilon}\right)(w_{\max}+\epsilon),
\end{equation}
where $\mathscr{C}^{\mathrm{opt}}$ is an optimal $K$-clustering for $\bV$ and $\tau(\cdot)$ is defined in~\eqref{eq: tau}.
\end{theorem}

The proof of Theorem~\ref{thm: original} is based on Corollary~\ref{coro: main} and various concentration bounds. A sketch is provided in Section~\ref{sec: proof_main_thm}  while the complete proof is provided in Appendix~\ref{prf: thm: original}. Note that both $\tau(\mathord{\cdot})$ and $\zeta(\mathord{\cdot})$ are continuous and monotonically increasing on $[0,\frac{1}{2}(K-1)]$. When the mixing weights are skewed (leading to a small $w_{\min}$), we require a strong separability assumption in~\eqref{eq: epsilon_cond}. This is consistent with Xiong {\em et al.}~\cite{xiong2009} which suggests that imbalanced clusters are more difficult to disambiguate for $k$-means. 
When $\delta_0$ is small (i.e., the  data is well-separated) and $N$ is large (so $\epsilon$ and $t$ can be chosen to be sufficiently small and large respectively), we have with probability close to $1$ that the upper bound on the ME distance given by \eqref{eq: res_origin} is close to~$0$.

When the ME distance between any optimal clustering for $k$-means $\mathscr{C}^{\mathrm{opt}}$ and the correct target clustering  $\mathscr{C}$ of the samples generated from a spherical GMM is small (and thus the implicit assumption of $k$-means is satisfied), we can readily perform $k$-means to find $\mathscr{C}^{\mathrm{opt}}$ to infer $\mathscr{C}$. The tightness of the upper bound in~\eqref{eq: res_origin} is assessed numerically in Section \ref{sec: num_exp}.

\begin{remark}
 {\em  The condition \eqref{eq: main_condition} can be considered as a separability assumption. In particular, when $K=2$, we have $\lambda_{\min}=w_1w_2 \|\bu_1 -\bu_2\|_2^2$ and \eqref{eq: main_condition} becomes
 \begin{equation}
  \|\bu_1 -\bu_2\|_2  >  \frac{\bar{\sigma}}{\sqrt{w_1 w_2\zeta(w_{\min})}}, \label{eqn:u1u2}
 \end{equation} 
 which is similar to \eqref{eq: well-sep}, the separability assumption of Vempala and Wang \cite{vempala2002}. Unlike \eqref{eq: well-sep}, in our separability assumption, there is no dependence on $F$, which is an advantage for datasets with a large number of features, such as the outputs of deep neural networks. There are also no implicit constants such as $C$ in  \eqref{eq: well-sep}.}
\end{remark}

\begin{remark}\label{remark:special_genK} {\em
 For general $K \ge 2$, we consider the following   example so that the separability condition in~\eqref{eq: main_condition} and the bound on the ME distance  in~\eqref{eq: res_origin} become more transparent. 
\begin{itemize}
 \item Assume that $w_i={1}/{K}$ for all $i \in [K]$; 
 \item Assume that there is a positive number  $d$ such that $\|\bu_i - \bar{\bu}\|_2 =  d$ for all $i \in [K]$ and there is an angle $\alpha \in [0,\pi]$ such that $(\bu_i - \bar{\bu})^T(\bu_j - \bar{\bu}) = d^2 \cos \alpha$ for all $i \ne j$. 
\end{itemize}
Then for this rather symmetrical setting, we have that 
\begin{itemize}
 \item $\cos \alpha = -{1}/(K-1)$; 
 \item $\|\bu_i - \bu_j\|_2^2 = 2d^2 (1-\cos \alpha) = 2d^2 {K}/(K-1)$ for all $i \ne j$;
 \item $\lambda_{\min} = \lambda_{K-1}(\bar{\bSigma}_0) = d^2/  (K-1) =  {\|\bu_i - \bu_j\|_2^2}/(2K)$ for all $i \ne j$. 
\end{itemize}
If we further consider the relaxation $\zeta(w_{\min}) >{w_{\min}}/{2}$ (this is due to the definition of $\zeta(\cdot)$ in \eqref{eqn:def_zeta}), the separability condition~\eqref{eq: main_condition} can be written as
\begin{equation}
 \frac{2K(K-1)\bar{\sigma}^2}{\|\bu_i - \bu_j\|_2^2} < \frac{1}{2K},
\end{equation}
or equivalently,
\begin{equation}
 \|\bu_i - \bu_j\|_2 > 2K\sqrt{K-1}\bar{\sigma}. \label{eqn:our_sep}
\end{equation}
For this special case, compared to the separability condition~\eqref{eq: well-sep} provided by Vempala and Wang~\cite{vempala2002}, there is no dependence on $F$ in our separability condition in \eqref{eqn:our_sep}. Moreover, \eqref{eqn:our_sep} matches the separability condition given by Kumar and Kannan~\cite{kumar2010} which is $\Omega({K\bar{\sigma}}/{\sqrt{w_{\min}}})$. However, note that the settings in \cite{kumar2010} and our paper  are significantly different. Kumar and Kannan~\cite{kumar2010} analyze a variant of the $k$-means algorithm, whilst we analyze optimal solutions of the sums-of-squares objective function (also see Remark~\ref{rmk:KK}).  It is thus reassuring that these two considerations lead to the same type of separability condition. Furthermore, note that for this simple example,
\begin{equation}
 \delta_0 = \frac{(K-1)^2 \bar{\sigma}^2}{d^2}. 
\end{equation}
If we leverage the inequality $\tau(\delta) < 2\delta$  (which follows from the definition of $\tau(\cdot)$ in \eqref{eq: tau}) and ignore the $\epsilon$ term in the upper bound provided in~\eqref{eq: res_origin}, we have that the upper bound for ME distance becomes 
\begin{equation}
d_{\mathrm{ME}}(\scC,\scC^{\mathrm{opt}})\le \tau(\delta_0) w_{\max} \le  \frac{2(K-1)^2 \bar{\sigma}^2}{Kd^2}.  \label{eqn:ub_ME}
\end{equation}  The upper bound in \eqref{eqn:ub_ME}  has an intuitive flavor.     It is proportional to $\bar{\sigma}^2$ (a measure of the spread of the intra-cluster points) and inversely proportional to $d^2$ (the squared distance between the component mean vectors  and their  average). Similar separability conditions and upper bounds for ME distance can be obtained for this   example in the theorems to follow.}
\end{remark}
\begin{remark}
{\em The separability condition in \eqref{eq: main_condition} is different from some other {\em pairwise separability} assumptions in the literature~\cite{dasgupta1999,dasgupta2000,arora2001,vempala2002,kalai2010}. Our condition is  a {\em global separability condition}. The intuitive reasons for this are twofold. First, we study the optimal solutions to the sum-of-squares distortion  measure in \eqref{eqn:sos}. This is a global measure, involving all clusters, and so we believe a global separability condition is natural. Second, we leverage several technical lemmas in the literature, such as Lemma \ref{lem: ME_dist}. These lemmas also involve global parameters such as $\lambda_{K-1} (\bS)$ and $\lambda_{K} (\bS)$, thus a global separability condition of the form of \eqref{eq: main_condition} seems unavoidable.}
\end{remark} 

\begin{remark}
{\em  The non-degeneracy condition is used to ensure that $\lambda_{\min}>0$. When $K=2$, to ensure that $\lambda_{\min}>0$, from~\eqref{eqn:u1u2}, we see that we only need to assume that the two component mean vectors are distinct. In particular, we do not require $\bu_1$ and $\bu_2$ to be linearly independent.}
\end{remark}

\begin{remark}\label{rmk:KK}
{\em The result by~Kumar and Kannan \cite{kumar2010} (discussed in Section~\ref{sec:gmm}) may, at a first glance, appear to be similar to Theorem \ref{thm: original} in the sense that both results show that under appropriate conditions, the $k$-means algorithm is a good choice for learning certain mixture models. However, there is a salient difference. The analysis of Kumar and Kannan \cite{kumar2010} is based on a variant of the $k$-means algorithm, while we only analyze the {\em objective function} of $k$-means (in~\eqref{eqn:sos}) which determines all optimal clusterings. Since there are multiple ways to approximately minimize the ubiquitous but intractable sum-of-squares distortion measure in~\eqref{eqn:sos}, our analysis is partly {\em algorithm-independent} (but dependent on the objective function in \eqref{eqn:sos}) and thus fundamental in the theory of clustering. Our analysis and theoretical results, in fact, {\em underpin why} the separability assumptions of various forms appear to be necessary to make theoretical guarantees for using $k$-means to learn mixture models.}
\end{remark}

Next, we show that under similar assumptions for the generating process and with a {\em weaker} separability assumption for the spherical GMM, any optimal clustering for the post-PCA dataset is also close  to the correct target clustering with high probability when $N$ is large enough.
\begin{theorem} \label{thm: afterPCA}
 Let the dataset $\bV\in \mathbb{R}^{F\times N}$ be generated under the same conditions given in Theorem~\ref{thm: original} with the separability assumption \eqref{eq: main_condition} being modified to
 \begin{equation}\label{eq: main_condition_afterPCA}
  \delta_1:= \frac{(K-1)\bar{\sigma}^2}{\lambda_{\min}+\bar{\sigma}^2 } < \zeta(w_{\min}).
 \end{equation}
 Let $\tilde{\bV} \in \mathbb{R}^{(K-1)\times N}$ be the post-$(K-1)$-PCA dataset of $\bV$. Then, for any   $\epsilon>0$ that satisfies 
 \begin{align} 
\epsilon &\le\min \left\{\frac{w_{\min}}{2}, \lambda_{\min}\!+\!\bar{\sigma}^2, (K \! -\! 1)\bar{\sigma}^2\right\} ,\;\;\mbox{and} \quad  \frac{(K-1)\bar{\sigma}^2+\epsilon}{\lambda_{\min} +\bar{\sigma}^2 - \epsilon}  \le \zeta(w_{\min}-\epsilon),  \label{eq: epsilon_cond1}
\end{align}
 and for any $t\ge 1$, when $N\ge CF^3 K^5 t^2/\epsilon^2$, where $C>0$ depends on  $\{  (w_k,\bu_k,\sigma_k^2)\}_{k\in[ K]}$, we have, with probability at least $1-165KF\exp\left(-t^2 K\right)$, 
\begin{equation}\label{eq: res_afterPCA}
 \mathrm{d}_{\mathrm{ME}}(\mathscr{C}, \tilde{\mathscr{C}}^{\mathrm{opt}}) \le \tau\left(\frac{(K-1)\bar{\sigma}^2+\epsilon}{\lambda_{\min}+\bar{\sigma}^2-\epsilon}\right)(w_{\max}+\epsilon),
\end{equation}
where $\mathscr{C}$ is the correct target clustering and $\tilde{\mathscr{C}}^{\mathrm{opt}}$ is an optimal $K$-clustering for $\tilde{\bV}$.
\end{theorem}
The proof of Theorem \ref{thm: afterPCA} is provided in Section \ref{sec:prfthm2}.
\begin{remark}\label{remark: SVD}
{\em  Vempala and Wang \cite{vempala2002} analyzed  the SVD of $\mathbf{\Sigma}_N$ (corresponding to PCA with no centering) instead of $\bar{\mathbf{\Sigma}}_N$ (corresponding to PCA). 
The proof of Corollary 3 in Vempala and Wang \cite{vempala2002} is based on the key observation that the subspace   spanned by the  first $K$ singular vectors of $\mathbf{\Sigma}_N$  lies close to the subspace spanned by the  $K$ component mean vectors of the spherical GMM with high probability when $N$ is   large. By performing $K$-SVD (cf.\ Section~\ref{sec: pca}) on $\mathbf{\Sigma}_N$, we have the following corollary. } \end{remark}
 \begin{corollary}
 Let the dataset $\bV\in \mathbb{R}^{F\times N}$ be generated under the same conditions given in Theorem~\ref{thm: original}. Let $\tilde{\bV}$ be the post-$K$-SVD dataset of $\bV$, then for any positive $\epsilon$ satisfying \eqref{eq: epsilon_cond} and for any $t\ge 1$, if $N\ge CF^3 K^5 t^2/\epsilon^2$, then with probability at least $1-167KF\exp\left(-t^2 K\right)$, the same upper bound   in \eqref{eq: res_origin} holds for $\mathrm{d}_{\mathrm{ME}}(\mathscr{C},\tilde{\mathscr{C}}^{\mathrm{opt}})$, where $\tilde{\mathscr{C}}^{\mathrm{opt}}$ is an optimal $K$-clustering for $\tilde{\bV}$.
 \end{corollary}

  Combining the results of Theorems~\ref{thm: original} and~\ref{thm: afterPCA}, by the triangle inequality for ME distance, we obtain the following corollary concerning an upper bound for $\mathrm{d}_\mathrm{ME}(\mathscr{C}^{\mathrm{opt}},\tilde{\mathscr{C}}^{\mathrm{opt}})$, the ME distance between any optimal clustering of the original dataset and any optimal clustering of the post-PCA dataset.
 \begin{corollary}
 Let the dataset $\bV\in \mathbb{R}^{F\times N}$ be generated under the same conditions given in Theorem~\ref{thm: original}. Let $\tilde{\bV}$ be the post-$(K-1)$-PCA dataset of $\bV$, then for any positive $\epsilon$ satisfying \eqref{eq: epsilon_cond} and for any $t\ge 1$, if $N\ge CF^5 K^5 t^2/\epsilon^2$, then with probability at least $1-201KF^2\exp\left(-t^2 K\right)$,  $\mathrm{d}_{\mathrm{ME}}(\mathscr{C}^{\mathrm{opt}},\tilde{\mathscr{C}}^{\mathrm{opt}})$ is upper bounded by the sum  of the right-hand-sides of~\eqref{eq: res_origin} and~\eqref{eq: res_afterPCA}.
 \end{corollary}

The proof of Theorem~\ref{thm: afterPCA} hinges mainly on the fact that the subspace spanned by the first $K-1$ singular vectors of $\bar{\mathbf{\Sigma}}_N$ is ``close'' to the subspace spanned by the first $K-1$ singular vectors of $\bar{\mathbf{\Sigma}}_0$.  See Lemma~\ref{lem: bd_bR} to follow for a precise statement. Note that the assumption in~\eqref{eq: main_condition_afterPCA} is weaker than~\eqref{eq: main_condition} and the upper bound given by \eqref{eq: res_afterPCA} is smaller than that in~\eqref{eq: res_origin} (if all the parameters are the same). In addition, when $K=2$, by applying PCA to the original dataset as described in Theorem~\ref{thm: afterPCA}, we obtain a $1$-dimensional dataset, which is easier to cluster  optimally compared to the $2$-dimensional dataset obtained by performing PCA with no centering as described in Remark~\ref{remark: SVD}. These comparisons also provide a theoretical basis for the fact that centering can result in a stark difference in PCA. 

\begin{figure*}[t]
\centering
\begin{tabular}{cc}
\includegraphics[width=0.475\columnwidth]{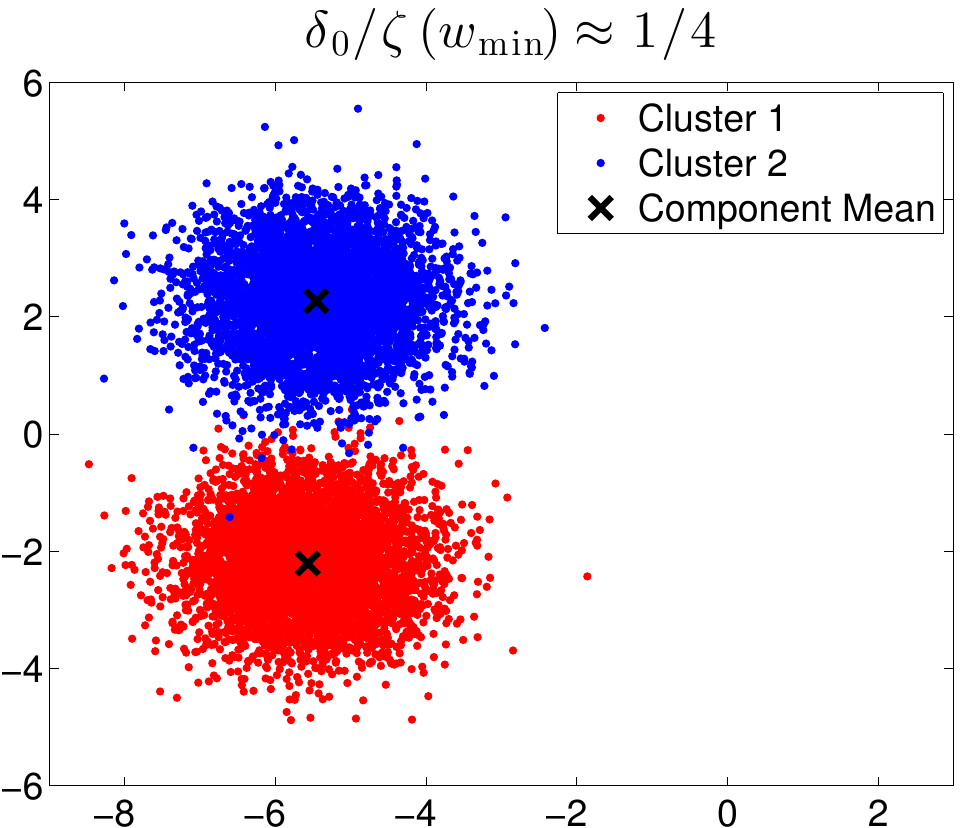}  &\hspace{.1in}
\includegraphics[width=0.475\columnwidth]{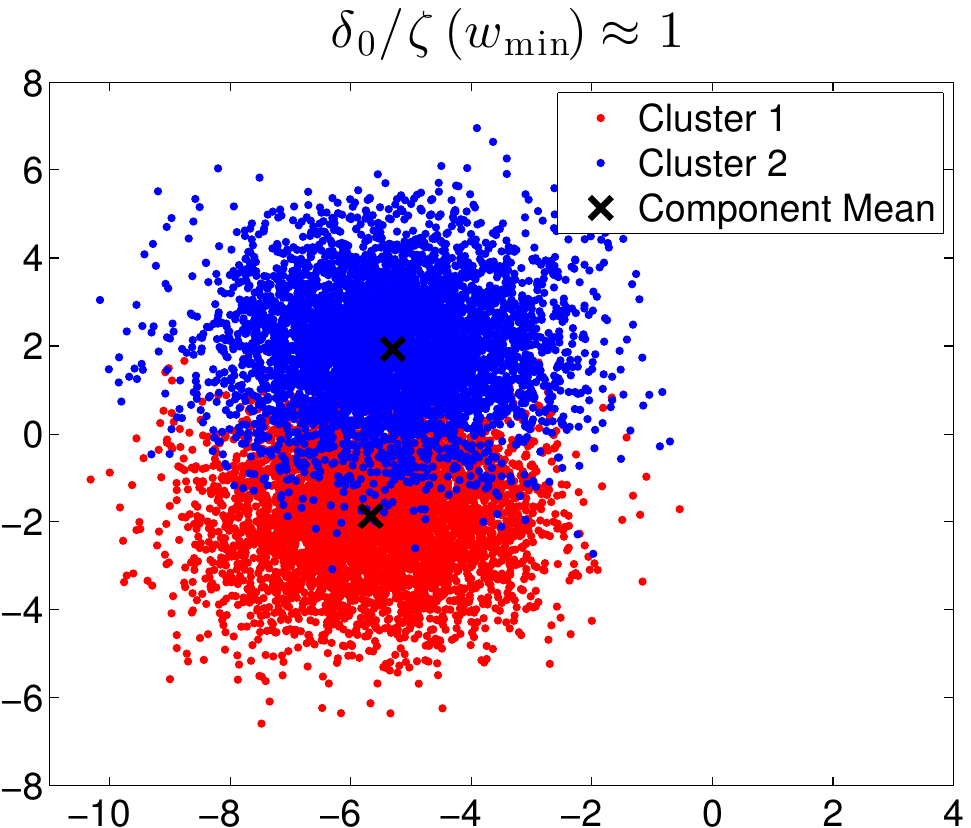}
\end{tabular}
\caption{Visualization of post-2-SVD datasets.}\label{fig: visualization}
\end{figure*}

\begin{figure*}[t]
\centering
\begin{tabular}{cc}
\subfloat{\includegraphics[width=0.5\columnwidth]{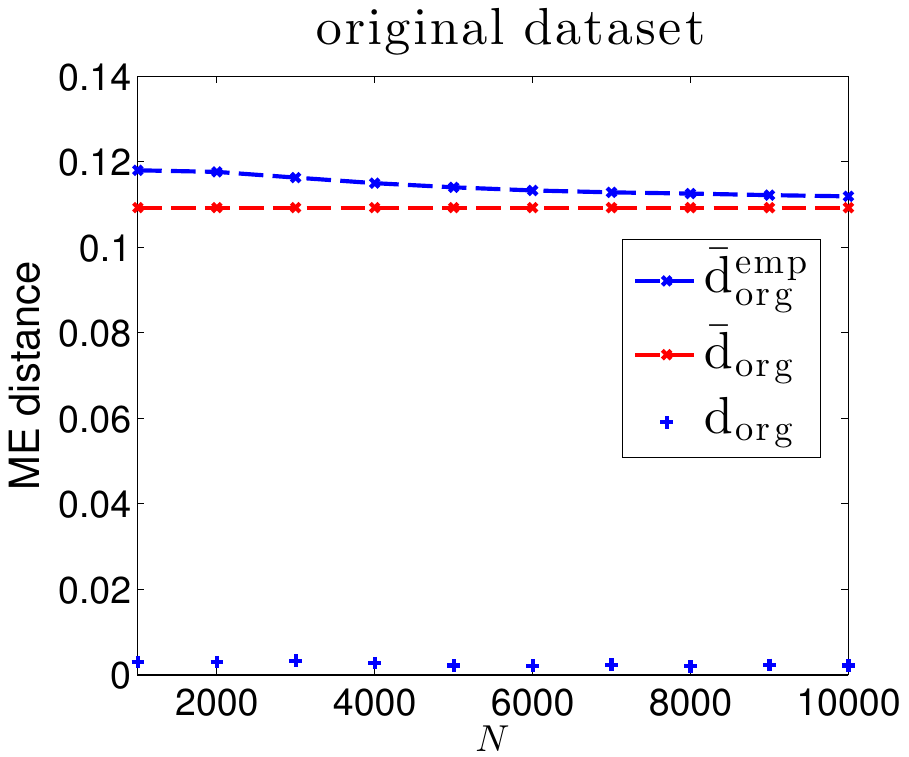}}  &\hspace{.1in}
\subfloat{\includegraphics[width=0.475\columnwidth]{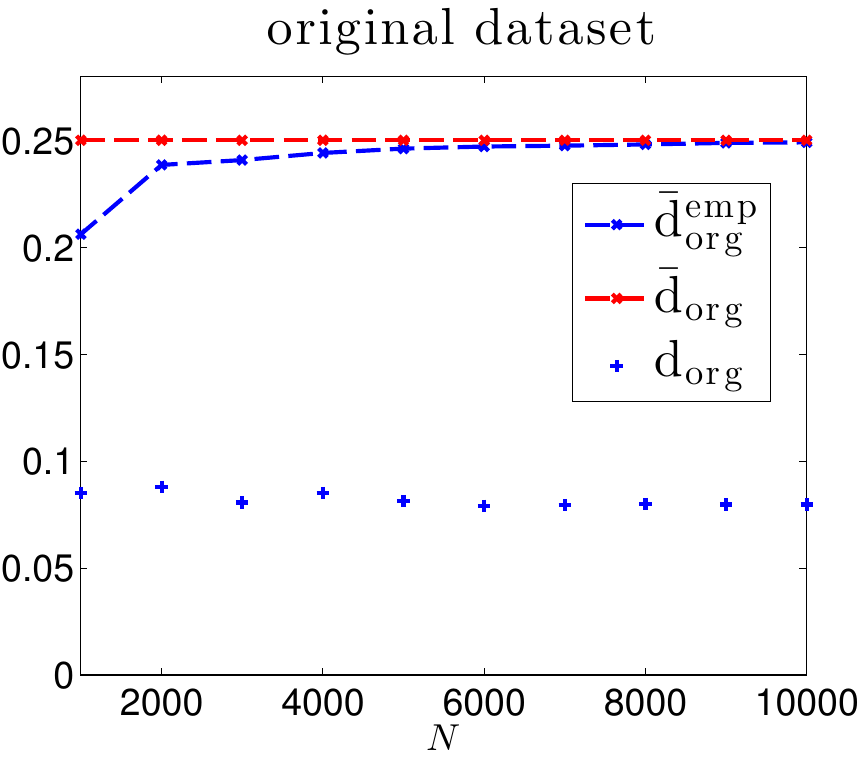}} \\
\subfloat{\includegraphics[width=0.5\columnwidth]{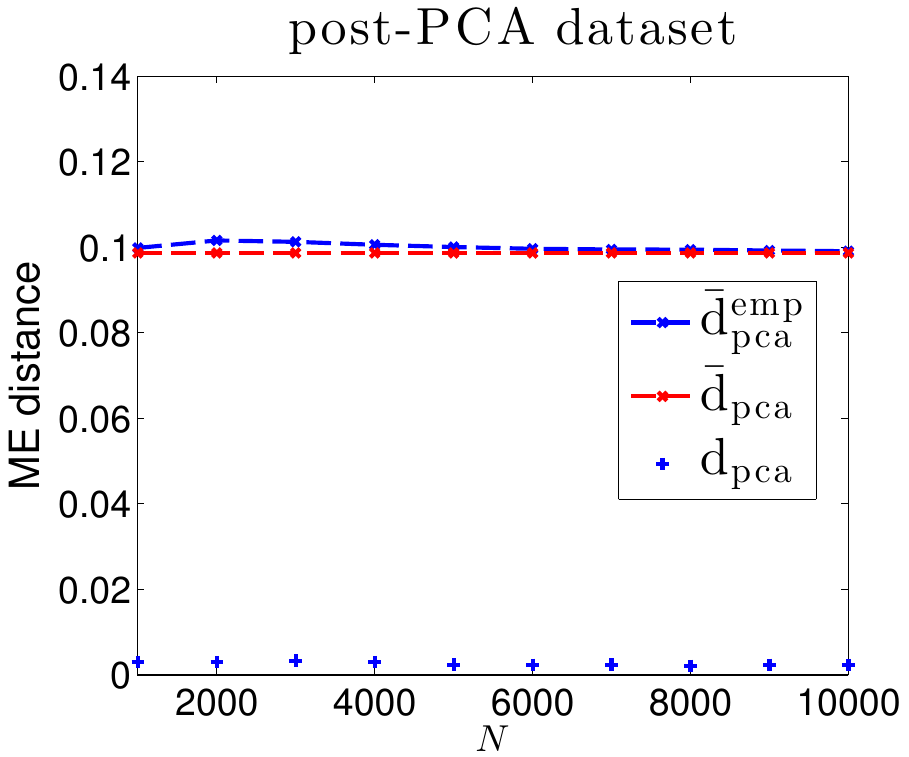}}  &\hspace{.1in}
\subfloat{\includegraphics[width=0.475\columnwidth]{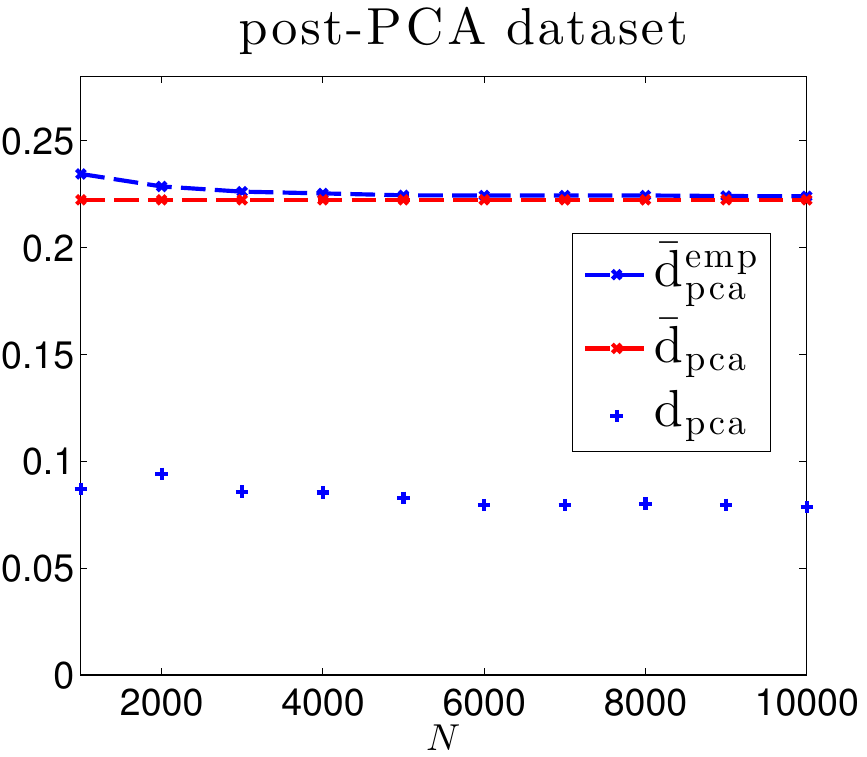}}  
\end{tabular}
\caption{True distances and their corresponding upper bounds.}\label{fig: bounds}
\end{figure*}

\subsection{Verifying Theorems \ref{thm: original} and \ref{thm: afterPCA}  numerically} \label{sec: num_exp}
To verify the correctness  and tightness of the upper bounds given in Theorems \ref{thm: original} and \ref{thm: afterPCA} and the efficacy of clustering post-PCA samples, we perform numerical experiments on synthetic datasets. We sample data points from two types of $2$-component spherical GMMs. The dimensionality of the data points is $F=100$, and the number of samples $N$ ranges from $1000$ to $10000$. Component mean vectors are randomly and uniformly picked from the hypercube $[0,1]^F$. Equal mixing weights are assigned to the components.
After fixing the mixing weights and the component mean vectors, $\lambda_{\min}$ is fixed. For all $k \in [K]$, we set the variances to be
\begin{align}
 \sigma_k^2 & = \frac{\lambda_{\min} \zeta(w_{\min}-\varepsilon)}{4(K-1)}, \quad \mbox{corr.\ to}\quad \frac{\delta_0}{\zeta(w_{\min})} \approx \frac{1}{4}, \quad\mbox{or} \label{eq: K2_well_sep} \\
 \sigma_k^2  &= \frac{\lambda_{\min} \zeta(w_{\min}-\varepsilon)}{K-1}, \quad\mbox{corr.\ to}\quad\frac{\delta_0}{\zeta(w_{\min})}\approx 1, \label{eq: K2_mod_sep}
\end{align}
where $\varepsilon = 10^{-6}$. In   all figures, left and right plots correspond to  \eqref{eq: K2_well_sep} and~\eqref{eq: K2_mod_sep} respectively.

We observe from Figure~\ref{fig: visualization} that for \eqref{eq: K2_well_sep}, the clusters are well-separated, while for \eqref{eq: K2_mod_sep}, the clusters are moderately well-separated. For both cases, the separability assumption \eqref{eq: main_condition} is satisfied. Similar to that in Boutsidis {\em et al.} \cite{boutsidis2015},  we use the command \texttt{kmeans(V', K, `Replicates', 10, `MaxIter', 1000)} in Matlab to obtain an approximately-optimal clustering of $\bV$. Here \texttt{V'} represents the transpose of $\bV$. This command means that we run $k$-means for 10 times with distinct initializations and pick the best outcome. For each run, the maximal number of iterations is set to be 1000. Define $\mathrm{d_{org}}:=\mathrm{d}_{\mathrm{ME}}(\mathscr{C},\mathscr{C}^{\mathrm{opt}})$ and define the (expected) upper bound for $\mathrm{d}_{\mathrm{ME}}(\mathscr{C},\mathscr{C}^{\mathrm{opt}})$ as $\mathrm{\bar{d}_{org}}:=\tau(\delta_0)w_{\max}$ (provided by Theorem~\ref{thm: original}). Similarly, we define $\mathrm{d_{pca}}:=\mathrm{d}_{\mathrm{ME}}(\mathscr{C},\tilde{\mathscr{C}}^{\mathrm{opt}})$ and the (expected) upper bound for $\mathrm{d}_{\mathrm{ME}}(\mathscr{C},\tilde{\mathscr{C}}^{\mathrm{opt}})$ is defined as $\mathrm{\bar{d}_{pca}}:=\tau(\delta_1)w_{\max}$ (given by Theorem~\ref{thm: afterPCA}). We use a superscript ``$\mathrm{emp}$'' to represent the corresponding empirical value. For example, $\delta_0^{\mathrm{emp}}:=\frac{\mathcal{D}(\bV, \mathscr{C})-\mathcal{D}^*(\bV)}{\lambda_{K-1}(\bS)-\lambda_{K}(\bS)}$ is an approximation of $\delta_0$ (calculated from the samples), and $\mathrm{\bar{d}_{org}^{emp}}:=\tau(\delta_0^{\mathrm{emp}})p_{\max}$ is an approximation of $\mathrm{\bar{d}_{org}}$, where $p_{\max}:=\max_k \frac{1}{N}|\mathscr{C}_k|$.  

Our numerical results are reported in Figure~\ref{fig: bounds}. We observe that the empirical values of upper bounds are close to the corresponding expected upper bounds. This observation verifies the correctness of the probabilistic estimates. For the well-separated case in \eqref{eq: K2_well_sep}, we observe that the upper bounds for ME distance are small compared to the moderately well-separated case in~\eqref{eq: K2_mod_sep}. For the former, the true distances $\mathrm{d_{org}}$ and $\mathrm{d_{pca}}$ are both close to $0$, even when the number of samples is 1000 (a small number in this scenario). The $k$-means algorithm can easily find an approximately-optimal clustering, which is also the approximately-correct clustering. For the moderately well-separated case, we observe that the upper bounds given in Theorem~\ref{thm: original} and Theorem~\ref{thm: afterPCA} are informative. In particular, they are only approximately $2.5$ times the corresponding true distances. 

From Figure~\ref{fig: runtime}, we observe that performing $k$-means for the original (high-dimensional) datasets is significantly slower than performing $k$-means for the corresponding post-PCA datasets (the reported running times for post-PCA datasets are the sums  of the running times for performing PCA and for performing $k$-means on the post-PCA datasets) when the number of samples is large. This difference is more pronounced for the   moderately well-separated case. For this case and $N=10000$, we have an  {\em order of magnitude speed up}. The running time for larger $N$ can be less than the running time for smaller $N$ because the number of iterations for $k$-means are possibly different.  All the results are averaged over $10$ runs. All experiments we run on an Intel Core i7 CPU at 2.50GHz and 16GB of memory, and the Matlab version is 8.3.0.532 (R2014a).

\begin{figure*}[t]
\centering
\begin{tabular}{cc}
\subfloat{\includegraphics[width=0.5\columnwidth]{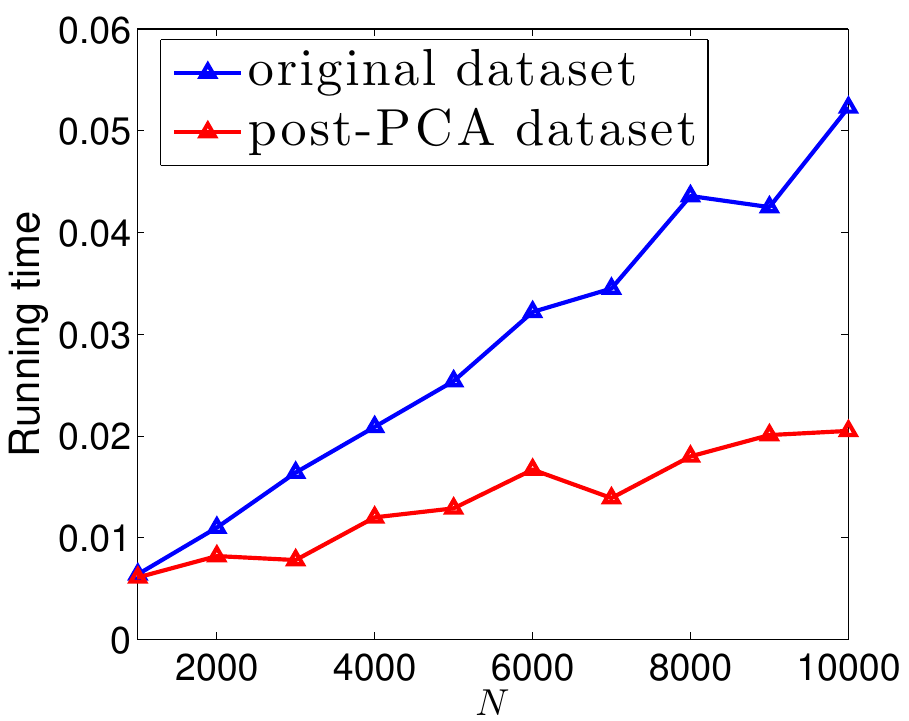}}  &\hspace{.1in}
\subfloat{\includegraphics[width=.475\columnwidth]{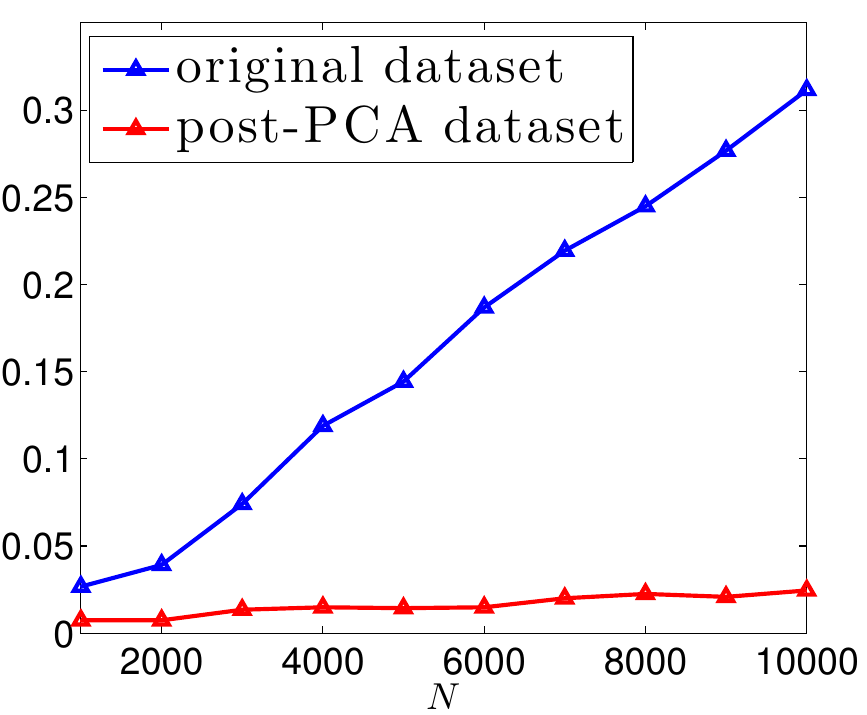}}  
\end{tabular}
\caption{Comparisons of running times in seconds. Notice the tremendous speed-up of clustering on the post-PCA dataset compared to the original one.}\label{fig: runtime}
\end{figure*}
\subsection{Useful Definitions and Lemmas} \label{sec: def_lem} 
First, we present the following lemma from~Golub and Van Loan \cite{golub2012} that provides an upper bound for perturbation of eigenvalues when the matrix is perturbed.
\begin{lemma}
\label{lem: svd_pert}
If $\bA$ and $\bA+\bE$ are in $\mathbb{R}^{M\times M}$,  then
\begin{equation}
  |\lambda_{m}(\bA+\bE)-\lambda_{m}(\bA)| \leq \|\bE\|_{2}
\end{equation}
for any $m\in [M]$ with $\lambda_{m}(\bA)$ being the $m$-th largest eigenvalue of $\bA$.
\end{lemma}

Because we will make use of the second-order moments of a mixture model, we present a simple lemma summarizing key facts.
\begin{lemma}\label{lem: moments_GMM}
 Let $\bx$ be a random  sample from a $K$-component mixture model with parameters $\{ (w_k,\bu_k,\mathbf{\Sigma}_k \}_{k\in [K]}$. Then,
\begin{align}
 \mathbb{E}\left(\bx\bx^T\right) &=\sum_{k=1}^{K}  w_k\left( \bu_k\bu_k^T +   \mathbf{\Sigma}_k\right)=\mathbf{\Sigma},
 \end{align}
 and
 \begin{align}
 & \mathbb{E}\left((\bx-\bar{\bu})(\bx-\bar{\bu})^T \right) =\sum_{k=1}^{K}  w_k \left( (\bu_k-\bar{\bu})(\bu_k-\bar{\bu})^T +  \mathbf{\Sigma}_k\right)=\bar{\mathbf{\Sigma}}.
\end{align} 
\end{lemma}

To apply Corollary~\ref{coro: main}, we need to ensure that $\lambda_{K-1}(\bS)-\lambda_{K}(\bS)$ which appears in the denominators of the expressions in~\eqref{eq: delta} is positive. Note that if we assume all the columns of data matrix $\bV$ are independently generated from a $K$-component spherical GMM, we have $\frac{1}{N}\left(\lambda_{K-1}(\bS)-\lambda_{K}(\bS)\right) \xrightarrow{\rmp} \lambda_{K-1}(\bar{\mathbf{\Sigma}}_0)$, where $\xrightarrow{\rmp}$ represents convergence in probability as $N\to\infty$. Under the non-degeneracy condition, $\lambda_K (\mathbf{\Sigma}_0) >0$. In addition, by the observation that $\bar{\mathbf{\Sigma}}_0 = \mathbf{\Sigma}_0 - \bar{\bu}\bar{\bu}^T$ and the following lemma in~Golub and Van Loan \cite{golub2012}, we have $\lambda_{\min}=\lambda_{K-1}(\bar{\mathbf{\Sigma}}_0) \ge \lambda_K (\mathbf{\Sigma}_0) >0$.  
\begin{lemma} \label{lem: rank1pert}
 Suppose $\bB=\bA+\tau \bv\bv^T$ where $\bA \in \mathbb{R}^{n\times n}$ is symmetric, $\bv$ has unit $2$-norm (i.e., $\|\bv\|_2=1$) and $\tau \in \mathbb{R}$.  Then,
\begin{equation}
  \lambda_i (\bB) \in\left\{  
  \begin{array}{cl}
   \left[\lambda_{i} (\bA),\lambda_{i-1} (\bA) \right] &\mbox{if  }\;\; \tau\ge 0,\, 2 \le i \le n\\
 \left[ \lambda_{i+1}(\bA) ,\lambda_{i}(\bA)  \right]  &\mbox{if  }\;\;  \tau\le 0,\, i \in [n-1] 
  \end{array}
  \right.   . \label{eqn:tau}
\end{equation} 
\end{lemma}
Furthermore, in order to obtain our probabilistic estimates, we need to make use of the following concentration bounds for sub-Gaussian and sub-Exponential random variables. The definitions and relevant lemmas are extracted from~Vershynin~\cite{vershynin2010}.

\begin{lemma} (Hoeffding-type inequality)  
\label{lem:large_dev_Gaussian}
A {\em sub-Gaussian random variable} $X$ is one such that $\left(\mathbb{E}|X|^{p}\right)^{1/p} \leq C  \sqrt{p}$ for some $C>0$ and for all $p\geq 1$.   Let $X_{1}, \ldots , X_{N}$ be independent zero-mean sub-Gaussian random variables, then for every $\ba=[a_1,a_2,\ldots,a_N]^T \in \mathbb{R}^N$ and every $t\ge 0$, it holds that
\begin{equation}
\mathbb{P}\left( \Big|\sum_{i=1}^{N}a_i X_{i}\Big| \ge t\right) \le   \exp\left(1-\frac{ct^2}{\|\ba\|_2^2}\right),
\end{equation}
where $c>0$ is a constant.
\end{lemma}

Typical examples of sub-Gaussian random variables are Gaussian, Bernoulli, and all bounded random variables. A random vector $\bx\in \mathbb{R}^{F}$ is called sub-Gaussian if $\bx^{T}\bz$ is a sub-Gaussian random variable for any deterministic vector $\bz\in \mathbb{R}^{F}$.

\begin{lemma} (Bernstein-type inequality)
\label{lem:large_dev}
A {\em sub-Exponential random variable} $X$ is one such that  $\left(\mathbb{E}|X|^{p}\right)^{1/p} \leq C  p$  for some $C>0$  and for all $p\geq 1$. Let $X_{1}, \ldots , X_{N}$ be independent zero-mean sub-Exponential random variables. 
It holds that
\begin{equation}
 \mathbb{P}\bigg( \Big|\sum_{i=1}^{N}X_{i}\Big|\ge \epsilon N\bigg)  \leq 2  \exp \left(-c \cdot \mathrm{min}\Big(\frac{\epsilon^{2}}{M^{2}},\frac{\epsilon}{M}\Big)N\right), \label{eqn:subexp}
\end{equation}
where $c>0$ is an absolute constant  and $M>0$ is the maximum of the sub-Exponential norms\footnote{The {\em sub-Exponential norm} of a sub-Exponential random variable $X$ is defined as $\|X \|_{\Psi_1} :=\sup_{p \ge 1} p^{-1}\left(\bbE|X|^p\right)^{1/p}$. }   of $\{X_i\}_{i=1}^N$,  i.e., $M =\max_{i\in [N]}  \|X_i \|_{\Psi_1}$.
\end{lemma}

The set of sub-Exponential random variables includes those that have tails heavier than Gaussian. It is easy to see that a sub-Gaussian random variable is also sub-Exponential. The following lemma, which can be found in~Vershynin \cite{vershynin2010}, is straightforward.  
\begin{lemma}  \label{lem:gau_exp}
A random variable $X$ is sub-Gaussian if and only if  $X^2$ is sub-Exponential.
\end{lemma}
Using this lemma, we see that Lemma~\ref{lem:large_dev} also provides a concentration bound for the sum of the squares of sub-Gaussian random variables. Finally, we can estimate empirical covariance matrices by the following lemma.  Note that in this lemma, we do not need to assume that the sub-Gaussian distribution $G$ in $\mathbb{R}^{F}$ has zero mean.
\begin{lemma} (Covariance estimation of sub-Gaussian distributions)  \label{lem:cov_est}
Consider a sub-Gaussian distribution $G$ in $\mathbb{R}^{F}$ with covariance matrix $\mathbf{\Sigma}$.  Define the empirical covariance matrix  $\bfSigma_{N}:=\frac{1}{N}\sum_{n=1}^{N}\bv_{n}\bv_{n}^{T}$ where each $\bv_{n}$ is an independent sample of $G$.  Let $\epsilon \in (0,1)$ and $ t\geq 1$.  If $N\geq C(t/\epsilon)^{2}F$ for some constant $C>0$, then with probability at least $1-2\exp(-t^{2}F)$,
\begin{equation} 
  \|\bfSigma_{N}-\bfSigma\|_{2} \leq \epsilon .
\end{equation}
\end{lemma}

\subsection{Proof Sketch of Theorem \ref{thm: original} } \label{sec: proof_main_thm} 
The general idea of the proof of Theorem \ref{thm: original} is to first estimate the terms defining $\delta$ in~\eqref{eq: delta} probabilistically. Next, we apply Corollary~\ref{coro: main} to show that any optimal clustering for the original data matrix is close to the correct target clustering corresponding to the spherical GMM. We provide the proof sketch below. Detailed calculations are deferred to Appendix~\ref{prf: thm: original}.
\begin{proof}[Proof Sketch of Theorem~\ref{thm: original}]
We estimate every term in~\eqref{eq: delta} for the correct target clustering $\mathscr{C}$. By Lemmas~\ref{lem:large_dev_Gaussian} and~\ref{lem:large_dev}, we have for any $\epsilon \in (0,1)$,
\begin{align}
 & \mathbb{P}\left(\Big|\frac{1}{N}\mathcal{D}(\bV,\mathscr{C})-F\bar{\sigma}^2 \Big| \ge \frac{\epsilon}{2} \right) \le 2K((e+2)F+2)\mathrm{exp}\left(-C_1 \frac{N \epsilon^2}{F^2 K^2}\right), \label{eqn:calD}
\end{align}
where $C_1>0$ depends on $\{  (w_k,\bu_k,\sigma_k^2)\}_{k\in[ K]}$.  See Appendix~\ref{prf: thm: original} for the detailed calculation of this and other inequalities in this proof sketch. In particular, for the justification of~\eqref{eqn:calD}, see the  steps leading to \eqref{eq: origin_part1}. These simply involve the triangle inequality, the union bound, and careful probabilistic estimates.  
Furthermore, if $N\ge C_2 F^5K^2 t^2/\epsilon^2$,  we have 
\begin{align}
 &\mathbb{P}\left(\Big|\frac{1}{N}\mathcal{D}^{*}(\bV)-(F-K+1)\bar{\sigma}^2 \Big| \ge \frac{\epsilon}{2} \right) \le  (F-K+1)(9FKe+2K)\exp(-t^2 F).  \label{eqn:D_star}
\end{align}
In addition, by Lemma~\ref{lem:cov_est},  for any $t\ge 1$, if $N \ge C_2  {F^3 K^2 t^2}/{\epsilon^2}$ (where $C_2>0$  also depends on $\{  (w_k,\bu_k,\sigma_k^2)\}_{k\in[ K]}$),   
\begin{eqnarray}\label{eq: bd_cov}
 \mathbb{P} \left(\|\bar{\mathbf{\Sigma}}_N-\bar{\mathbf{\Sigma}}\|_2 \ge \frac{\epsilon}{2} \right) \le (9FKe+2K)\exp\left(-t^2 F\right). \label{eqn:SigmaN}
\end{eqnarray}
 Therefore, by the matrix perturbation inequalities in Lemma~\ref{lem: svd_pert},  when $N \ge C_2  {F^3K^2 t^2}/{\epsilon^2}$, we have
\begin{align}
 & \mathbb{P} \left(\Big|\frac{1}{N}\lambda_{K-1}(\bS)-\left(\lambda_{\min}+\bar{\sigma}^2\right)\Big| \ge \frac{\epsilon}{2} \right) \le (9FKe+2K)\exp\left(-t^2 F\right), \label{eqn:lambda_Kminus1}
 \end{align}
 and
 \begin{align}
 \mathbb{P}\left(\Big|\frac{1}{N}\lambda_{K}(\bS)-\bar{\sigma}^2 \Big| \ge \frac{\epsilon}{2} \right) & \le (9FKe+2K)\exp\left(-t^2 F\right). \label{eqn:lambda_KS}
\end{align}
Combining these results, appealing to Corollary~\ref{coro: main}, the union bound, and the property that both $\tau(\mathord{\cdot})$ and $\zeta(\mathord{\cdot})$ are continuous and monotonically increasing functions on $[0,\frac{1}{2}(K-1)]$, we obtain \eqref{eq: res_origin} as desired. 
\end{proof}
\subsection{Proof of Theorem \ref{thm: afterPCA}}\label{sec:prfthm2}
Here, we prove Theorem~\ref{thm: afterPCA}. Following the notations in Section~\ref{sec: pca}, we write $\bar{\mathbf{\Sigma}}_N=\bP\bD\bP^T$, $\bP_{K-1}=\bP(\colon,1\colon K-1)$, and $\bP_{-(K-1)} = \bP(\colon,K\colon F)$. We also  denote $\tilde{\bV}=\bP_{K-1}^T \bV$ as the post-$(K-1)$-PCA dataset of $\bV$. Instead of using $\bP_{K-1}$ which is correlated to  the samples, we consider the SVD of $\bar{\mathbf{\Sigma}}_0$ and project the original data matrix onto $\mathbb{R}^{K-1}$ using  the first $K-1$ singular vectors of $\bar{\mathbf{\Sigma}}_0$.  That is, if we write the SVD of $\bar{\mathbf{\Sigma}}_0$ as $\bar{\mathbf{\Sigma}}_0=\bQ_{K-1} \bE_{K-1} \bQ_{K-1}^T$, we first analyze $\hat{\bV} := \bQ_{K-1}^T \bV$, before relating these results to those of $\tilde{\bV}=\bP_{K-1}^T \bV$.  We can similarly estimate the terms in \eqref{eq: delta} for the corresponding $(K-1)$-dimensional spherical GMM. Furthermore, we estimate the difference between the results obtained from projecting the original data matrix onto $\mathbb{R}^{K-1}$ using  the first $K-1$ singular vectors of $\bar{\mathbf{\Sigma}}_0$ and the results obtained from projecting the original data matrix onto $\mathbb{R}^{K-1}$ using the columns of $\bP_{K-1}$.
\begin{proof}[Proof of Theorem~\ref{thm: afterPCA}:]
 By the non-degeneracy condition and Lemma~\ref{lem: rank1pert}, we have $\mathrm{rank}(\bar{\mathbf{\Sigma}}_0)=K-1$. Let the compact SVD of $\bar{\mathbf{\Sigma}}_0$ be 
 \begin{equation}\label{eq: basic_svd}
  \bar{\mathbf{\Sigma}}_0=\bQ_{K-1} \bE_{K-1} \bQ_{K-1}^T,
 \end{equation}
 where $\bQ_{K-1} \in \mathbb{R}^{F\times (K-1)}$ has orthonormal columns and $\bE_{K-1} \in \mathbb{R}^{(K-1)\times (K-1)}$ is a diagonal matrix. Since $\bQ_{K-1}^T \bQ_{K-1} = \bI$, by the property of Gaussians, we know if $\bx$ is a random vector with a spherical Gaussian distribution $\mathcal{N}(\bu,\sigma^2\bI)$ in $\mathbb{R}^F$, then $\bQ_{K-1}^T \bx$ is a random vector with a spherical Gaussian distribution $\mathcal{N}(\bQ_{K-1}^T\bu,\sigma^2\bI)$ in $\mathbb{R}^{K-1}$. Let $\hat{\bV}:=\bQ_{K-1}^T \bV$, $\hat{\bZ}$ be the centralized matrix of $\hat{\bV}$ and $\hat{\bS}:=\hat{\bZ}^T\hat{\bZ}$. Denote $\hat{\bar{\bu}}=\bQ_{K-1}^T \bar{\bu}$  and $\hat{\bu}_k=\bQ_{K-1}^T \bu_k$  for all $k \in [K]$. Let $\hat{\bar{\mathbf{\Sigma}}}_0:=\sum_{k=1}^{K} w_k (\hat{\bu}_k-\hat{\bar{\bu}})(\hat{\bu}_k-\hat{\bar{\bu}})^T$. Define $\bX:=[\sqrt{w_1}(\bu_1-\bar{\bu}),\ldots,\sqrt{w_K}(\bu_K-\bar{\bu})] \in \mathbb{R}^{F\times K}$ and let $\hat{\bX}:=\bQ_{K-1}^T \bX$. Select $\bQ_{-(K-1)} \in \mathbb{R}^{F\times (F-K+1)}$ such that $[\bQ_{K-1},\bQ_{-(K-1)}]$ is an orthogonal matrix. We have
 \begin{align}
    \bX^T \bX - \hat{\bX}^T\hat{\bX} &= \bX^T \bX - \bX^T (\bQ_{K-1} \bQ_{K-1}^T) \bX = \bX^T \bQ_{-(K-1)} \bQ_{-(K-1)}^T \bX = 0.
 \end{align}
 Thus we have 
 \begin{align}
  \lambda_{\min} &= \lambda_{K-1}(\bX\bX^T)= \lambda_{K-1}(\bX^T\bX) = \lambda_{K-1}(\hat{\bX}^T\hat{\bX})=\lambda_{K-1}(\hat{\bar{\mathbf{\Sigma}}}_0).
 \end{align}
 Then similar to \eqref{eqn:calD} in the proof sketch of Theorem~\ref{thm: original}, for any $\epsilon \in (0,1)$,
 \begin{align}
  &\mathbb{P}\left(\Big|\frac{1}{N}\mathcal{D}(\hat{\bV},\mathscr{C})- (K-1)\bar{\sigma}^2 \Big| \ge \frac{\epsilon}{2}\right) \le 2K((e+2)K+e)\exp\left(-C_3 \frac{N\epsilon^2}{K^4}\right). \label{eq: distortion_afterPCA}
 \end{align}
In addition, similar to \eqref{eqn:lambda_Kminus1} in the proof sketch of Theorem~\ref{thm: original}, for any $t \ge 1$, if $N\ge C_4 {K^5t^2}/{\epsilon^2}$, 
 \begin{align}
  & \mathbb{P}\left(\Big|\frac{1}{N}\lambda_{K-1}(\hat{\bS})-(\lambda_{\min}+\bar{\sigma}^2) \Big|\ge \frac{\epsilon}{2}\right) \le 9(eK^2+2K)e\exp(-t^2 K),\label{eq: lambda_afterPCA}
 \end{align}
 where $C_3, C_4>0$ depend on $\{  (w_k,\bu_k,\sigma_k^2)\}_{k\in[ K]}$.
 Note that
$
  \frac{1}{N}\mathcal{D}^{*}(\hat{\bV}) = \frac{1}{N}\lambda_{K}(\hat{\bS})=0.
$
 Thus, we only need to estimate $\frac{1}{N} \big|\mathcal{D}(\hat{\bV},\mathscr{C})-\mathcal{D}(\tilde{\bV},\mathscr{C})\big|$ and $\frac{1}{N}\big|\lambda_{K-1}(\hat{\bS})-\lambda_{K-1}(\tilde{\bS})\big|$, where $\tilde{\bS}:=\tilde{\bZ}^T\tilde{\bZ}$ and $\tilde{\bZ}$ is the centralized matrix of $\tilde{\bV}$. By \eqref{eq: distortion_basis} and writing $\bR:=\bQ_{K-1} \bQ_{K-1}^T-\bP_{K-1} \bP_{K-1}^T$,  we have
 \begin{align}
  \frac{1}{N}\big|\mathcal{D}(\hat{\bV},\mathscr{C})-\mathcal{D}(\tilde{\bV},\mathscr{C})\big| &= \left|\left\langle \frac{1}{N}\left(\bV \bV^T- \bV\bar{\bH}^T\bar{\bH} \bV^T\right), \bR \right\rangle \right|  \\
&  \le \left(1+\|\bar{\bH}^T\|_\mathrm{F}^2\right) \left(\frac{1}{N}\|\bV\|_\mathrm{F}^2\right) \|\bR\|_\mathrm{F} \\
&  = (1+K) \left(\frac{1}{N}\|\bV\|_\mathrm{F}^2\right) \|\bR\|_\mathrm{F}. \label{eq: dist_diff} 
\end{align}
Note that 
\begin{equation}
 \mathbb{E}\left[\frac{1}{N}\|\bV\|_\mathrm{F}^2\right] = \sum_{k=1}^{K} w_k \left(\|\bu_k\|_2^2+F\sigma_k^2\right).
\end{equation}
In addition, by Lemma~\ref{lem: svd_pert} and routine calculations,
\begin{align}
    \frac{1}{N}\big|\lambda_{K-1}(\hat{\bS})-\lambda_{K-1}(\tilde{\bS})\big| & \le \left\|\frac{1}{N}(\hat{\bS}-\tilde{\bS})\right\|_2 \\
    &\le \left\|\frac{1}{N}(\hat{\bS}-\tilde{\bS})\right\|_\mathrm{F} = \left\|\frac{1}{N}\bZ^T \bR \bZ\right\|_\mathrm{F} \\
 &\le \frac{1}{N} \|\bR\|_\mathrm{F} \|\bZ\|_\mathrm{F}^2 \\
  & \le  \|\bR\|_\mathrm{F} \left(\frac{1}{N}\|\bV\|_\mathrm{F}^2+\|\bar{\bv}\|_2^2\right). \label{eq: eigen_diff} 
\end{align}
Thus in \eqref{eq: dist_diff} and \eqref{eq: eigen_diff}, we   need to bound $\|\bR\|_\mathrm{F}$. According to \eqref{eq: bd_cov}, $\|\bar{\mathbf{\Sigma}}_N-\bar{\mathbf{\Sigma}}\|_2$ can be bounded probabilistically. By Lemma~\ref{lem: bd_bR} to follow, the upper bound of $\|\bR\|_\mathrm{F}$ can be deduced by the upper bound of $\|\bar{\mathbf{\Sigma}}_N-\bar{\mathbf{\Sigma}}\|_2$. By leveraging additional concentration bounds for sub-Gaussian and sub-Exponential distributions given in Lemmas~\ref{lem:large_dev_Gaussian} and~\ref{lem:large_dev}, we deduce that if $N\ge C_5  {F^3 K^5 t^2}/{\epsilon^2}$, 
\begin{align}
  \mathbb{P} \left(\frac{1}{N} \big|\mathcal{D}(\hat{\bV},\mathscr{C})-\mathcal{D}(\tilde{\bV},\mathscr{C})\big| \ge \frac{\epsilon}{2}\right)& \le 48KF\exp(-t^2 F), \label{eqn:boudn1} \\
 \mathbb{P} \left(\frac{1}{N}\big|\lambda_{K-1}(\hat{\bS})-\lambda_{K-1}(\tilde{\bS})\big| \ge \frac{\epsilon}{2}\right) &\le 48KF\exp(-t^2 F), \label{eqn:boudn2}
\end{align}
where $C_5>0$ depends on $\{  (w_k,\bu_k,\sigma_k^2)\}_{k\in[ K]}$. The proofs of~\eqref{eqn:boudn1} and \eqref{eqn:boudn2} are omitted for the sake of brevity. Combining these bounds with \eqref{eq: distortion_afterPCA} and \eqref{eq: lambda_afterPCA} and  by using Corollary~\ref{coro: main}, we obtain \eqref{eq: res_afterPCA} as desired.   
\end{proof}

The following is a lemma essential for establishing upper bounds of \eqref{eq: dist_diff} and \eqref{eq: eigen_diff} in the proof of Theorem \ref{thm: afterPCA}. Note that if we view the Grassmannian manifold as a metric measure space, the distance between subspaces $\mathcal{E}$ and $\mathcal{F}$ can be defined as~\cite{vershynin2016high}
\begin{equation} \label{eq: dist_subspaces}
 \rmd_{\calS} (\mathcal{E}, \mathcal{F}):=\|\bfP_{\mathcal{E}}-\bfP_{\mathcal{F}}\|_\rmF,
\end{equation}
where $\bfP_{\mathcal{E}}$ and $\bfP_{\mathcal{F}}$ are the orthogonal projections onto $\calE$ and $\calF$. Because $\bQ_{K-1} \bQ_{K-1}^T$ and $\bP_{K-1} \bP_{K-1}^T$ are the orthogonal projection matrices for projections onto the subspaces spanned by the columns of $\bQ_{K-1}$ and $\bP_{K-1}$ respectively, $\|\bR\|_{\mathrm{F}}$ is a measure of the distance between these two subspaces.

\begin{lemma}\label{lem: bd_bR}
 For $\epsilon>0$, if $\|\bar{\mathbf{\Sigma}}_N-\bar{\mathbf{\Sigma}}\|_2 \le \epsilon$, then 
 \begin{equation}
  \|\bR\|_{\mathrm{F}} \le \frac{4\sqrt{K}\epsilon}{\lambda_{\min}}.
 \end{equation}
\begin{proof}
 By Lemma \ref{lem: svd_pert}, $|\lambda_{F}(\bar{\mathbf{\Sigma}}_N)-\bar{\sigma}^2|=|\lambda_{F}(\bar{\mathbf{\Sigma}}_N)-\lambda_{F}(\bar{\mathbf{\Sigma}})| \le \epsilon$. Then
\begin{equation}
 \|\bar{\mathbf{\Sigma}}_N -\lambda_{F}(\bar{\mathbf{\Sigma}}_N)\bI-\bar{\mathbf{\Sigma}}_0\|_2 \le 2\epsilon.
\end{equation}
Because $\bar{\mathbf{\Sigma}}_N -\lambda_{F}(\bar{\mathbf{\Sigma}}_N)\bI$ is also positive semidefinite, the SVD is
\begin{equation}
 \bar{\mathbf{\Sigma}}_N -\lambda_{F}(\bar{\mathbf{\Sigma}}_N)\bI = \bP\bar{\bD}\bP^T,
\end{equation}
where $\bar{\bD}:=\bD-\lambda_{F}(\bar{\mathbf{\Sigma}}_N)\bI$. Let $\bar{\bD}_{K-1}=\bar{\bD}(1\colon K-1,1\colon K-1)$. Note that  $\mathrm{rank}(\bar{\mathbf{\Sigma}}_0)=K-1$. By Lemma~\ref{lem: svd_pert}, $\lambda_{K}(\bP\bar{\bD}\bP^T) \le 2\epsilon$ and thus we have
\begin{equation}
 \|\bP\bar{\bD}\bP^T-\bP_{K-1} \bar{\bD}_{K-1} \bP_{K-1}^T\|_2 = \lambda_{K}(\bP\bar{\bD}\bP^T) \le 2\epsilon.
\end{equation}
Therefore, there exists a matrix $\bE_0$ with $\|\bE_0\|_2 \le 4\epsilon$, such that
\begin{equation}
 \bar{\mathbf{\Sigma}}_0 = \bP_{K-1} \bar{\bD}_{K-1} \bP_{K-1}^T + \bE_0.
\end{equation}
That is,
\begin{equation}\label{eq: subspace_comp}
 \bQ_{K-1} \bE_{K-1} \bQ_{K-1}^T = \bP_{K-1} \bar{\bD}_{K-1} \bP_{K-1}^T + \bE_0.
\end{equation}
Recall that $\bP_{-(K-1)}:=\bP(:,K:F)$. We obtain
\begin{align}
\bQ_{K-1} \bQ_{K-1}^T  &= \bE_0 \bQ_{K-1} \bE_{K-1}^{-1} \bQ_{K-1}^T + \bP_{K-1} \bar{\bD}_{K-1} \bP_{K-1}^T \bQ_{K-1} \bE_{K-1}^{-1} \bQ_{K-1}^T \\
 & = \bP_{K-1} \bP_{K-1}^T  \bP_{K-1} \bar{\bD}_{K-1} \bP_{K-1}^T  \bar{\mathbf{\Sigma}}_0^{+} \!+\! \bE_0 \bar{\mathbf{\Sigma}}_0^{+} \\
 & = \bP_{K-1} \bP_{K-1}^T + \bP_{-(K-1)} \bP_{-(K-1)}^T \bE_0 \bar{\mathbf{\Sigma}}_0^{+},
\end{align}
where $\bar{\mathbf{\Sigma}}_0^{+}:=\bQ_{K-1} \bE_{K-1}^{-1} \bQ_{K-1}^T$ is the Moore-Penrose generalized inverse (or pseudoinverse) of $\bar{\mathbf{\Sigma}}_0$ and its largest eigenvalue is $ \lambda_{\min}^{-1}$. Because $\bP_{-(K-1)} \bP_{-(K-1)}^T$ projects vectors in $\mathbb{R}^{F}$ onto the linear space spanned by the orthonormal columns of $\bP_{-(K-1)}$, we have
\begin{align}
   \|\bR\|_{\mathrm{F}} &= \|\bQ_{K-1} \bQ_{K-1}^T - \bP_{K-1} \bP_{K-1}^T\|_\mathrm{F} \\
   & \le \|\bE_0 \bar{\mathbf{\Sigma}}_0^{+}\|_\mathrm{F} \le \|\bE_0\|_2 \sqrt{K} \|\bar{\mathbf{\Sigma}}_0^{+}\|_2 
  \le \frac{4\sqrt{K}\epsilon}{\lambda_{\min}},
\end{align}
where we use the inequality that $\|\bM\bN\|_{\mathrm{F}} \le \|\bM\|_2 \|\bN\|_{\mathrm{F}}$ for any two compatible matrices $\bM$ and $\bN$. The inequality $\|\bar{\mathbf{\Sigma}}_0^{+}\|_\mathrm{F} \le \sqrt{K} \|\bar{\mathbf{\Sigma}}_0^{+}\|_2$ arises because $\bar{\mathbf{\Sigma}}_0^{+}$  only contains  $K-1$ positive eigenvalues.
\end{proof}
\end{lemma}
\begin{remark}\label{remark: close_subspaces} {\em
By Lemmas~\ref{lem:cov_est} and~\ref{lem: bd_bR}, we see that the subspace spanned by the first $K-1$ singular vectors of $\bar{\mathbf{\Sigma}}_0$ lies close to the subspace spanned by the first $K-1$ singular vectors of $\bar{\mathbf{\Sigma}}_N$ when the number of samples is sufficiently large. Note that $\sum_{k=1}^K w_k (\bu_k-\bar{\bu}) (\bu_k-\bar{\bu})^T =\bar{\mathbf{\Sigma}}_0=\bQ_{K-1} \bE_{K-1} \bQ_{K-1}^T$. We also obtain that the subspace spanned by $\bu_k-\bar{\bu}, k\in [K]$ is close to the subspace spanned by the first $K-1$ singular vectors of $\bar{\mathbf{\Sigma}}_N$. 
 Note that $\mathbf{\Sigma}_0 = \sum_{k=1}^K w_k \bu_k \bu_k^T$. A similar result  can be obtained for $K$-SVD to corroborate an observation by Vempala and Wang \cite{vempala2002} (cf.\ Remark~\ref{remark: SVD}).}
\end{remark}

\section{Extension to Mixtures of Log-Concave Distributions} \label{sec: main_thm_log} 

In this section, we extend the results in Section \ref{sec: main_thm}  to mixtures of {\em log-concave distributions}. This is motivated  partly by Arora and Kannan \cite{arora2001} who mention the algorithm they design for analyzing the learning of GMMs may be extended to log-concave distributions (besides Gaussians). Also, Kannan {\em et al.} \cite{kannan2005} generalize the work of Vempala and Wang \cite{vempala2002} from spherical GMMs to mixtures of log-concave distributions. Furthermore, Brubaker \cite{brubaker2009robust} considers the robust learning of mixtures of log-concave distributions. We also consider the learning of mixtures of log-concave distributions although the structure of the theoretical analysis  is mostly similar to that for spherical GMMs.  We provide some necessary preliminary definitions and results for log-concave distributions and random variables in Section \ref{sec: des_main_thm_log1}. To prove our main results concerning such distributions, we need to employ a slightly different proof strategy vis-\`a-vis the one for spherical GMMs. We discuss these differences in Section~\ref{sec: proof_main_diff_log}. After presenting some preliminary lemmas, we present our main theorems for mixtures of log-concave distributions in Section~\ref{sec: des_main_thm_log}. 

\subsection{Definitions and Useful Lemmas} \label{sec: des_main_thm_log1} 
A function $f: \bbR^F \rightarrow \bbR_{+}$ is {\em log-concave} if its logarithm $\log f$ is concave. That is, for any two vectors $\bx,\by \in \bbR^F$ and any $\alpha \in [0,1]$, 
\begin{equation}
 f(\alpha \bx + (1-\alpha)\by) \ge f(\bx)^\alpha f(\by)^{1-\alpha}.
\end{equation}
A distribution is {\em log-concave} if its probability density (or mass) function is log-concave. We say a random variable is {\em log-concave} if its distribution is log-concave. There are many distributions that are log-concave, including Gaussian distributions, exponential distributions, and Laplace distributions.  In particular, distributions that belong to exponential families are log-concave.  Log-concave distributions have several desirable properties. For example, the sum of two independent log-concave random variables (i.e., the convolution of two log-concave distributions) is also log-concave. In addition, the linear projection of a log-concave distribution onto a lower-dimensional space  remains log-concave. To start off, we need to estimate the  deviation of an empirical covariance matrix from the true covariance matrix. We leverage the following lemma due to~Kannan {\em et al.}~\cite{kannan2005}.
\begin{lemma} \label{lem: logconcave_cov}
 Let $\epsilon,\eta \in (0,1)$, and $\by_1,\by_2,\ldots,\by_N$ be zero-mean i.i.d.\ random vectors from a log-concave distribution in $\mathbb{R}^F$. Then there is an absolute constant $C>0$ such that if $N > C\frac{F}{\epsilon^2}\log^5\left(\frac{F}{\epsilon \eta}\right)$, with probability at least $1 - \eta$, $\forall\, \bv \in \mathbb{R}^F$,
 \begin{equation}\label{eq: var_cov_log}
  (1-\epsilon) \mathbb{E}\big[ (\bv^T\by )^2 \big] \le \frac{1}{N}\sum_{n=1}^N (\bv^T\by_n)^2 \le (1+\epsilon) \mathbb{E}\big[ (\bv^T\by )^2 \big].
 \end{equation}
\end{lemma}

Note that Lemma \ref{lem: logconcave_cov} provides an estimate for the empirical covariance matrix. This is because that for any symmetric matrix $\bM$, $\|\bM\|_2=\sup_{\bv \in \bbR^F, \|\bv\|_2=1} |\langle \bM \bv, \bv\rangle|$ {and} \eqref{eq: var_cov_log} is equivalent to 
\begin{equation}
 \|\bfSigma_N - \bfSigma\|_2 \le \epsilon \|\bfSigma\|_2.
\end{equation}

Using Lemma~\ref{lem: logconcave_cov}, we also have the following corollary which provides a useful concentration bound for the sum of the squares of log-concave random variables.
\begin{corollary}\label{coro: concentration_log_sum_square}
 Let $\epsilon \in (0,1)$, and $y_1,y_2,\ldots,y_N$ be i.i.d.\ samples from a log-concave distribution with expectation $\mu$ and variance $\sigma^2$. Then if $\epsilon$ is sufficiently small, there is an absolute constant $C>0$ such that for $N>C\frac{1}{\epsilon^2}\log^5\left(\frac{1}{\epsilon \eta}\right)$, with probability at least $1 - \eta$,
 \begin{equation}\label{eq: square_log}
  \left|\frac{1}{N} \sum_{n=1}^N y_n^2 - (\mu^2 + \sigma^2)\right| \le \epsilon.
 \end{equation}
\end{corollary}
The proof of the above corollary  follows directly from  Lemma \ref{lem: logconcave_cov}  and so it is omitted. We will make use of the following lemma \cite{golub2012} concerning the eigenvalues  of a matrix which is the  sum of two matrices. 
\begin{lemma} \label{lem: bd_eigen_sum}
 If $\bA$ and $\bA+\bE$ are both $n$-by-$n$ symmetric matrices, then  for all $k\in [n]$,
 \begin{align}
  \lambda_k(\bA) + \lambda_n(\bE) \le \lambda_k(\bA+\bE) \le \lambda_k(\bA)& + \lambda_1(\bE).
 \end{align}
 \end{lemma}
 
\subsection{Main Differences in Proof Strategy vis-\`a-vis Spherical GMMs} \label{sec: proof_main_diff_log} 
For non-spherical GMMs or more general mixtures of log-concave distributions, we cannot expect a result similar to that mentioned in Remarks~\ref{remark: SVD} and~\ref{remark: close_subspaces}. That is, in general, the subspace   spanned by the top $K$ singular vectors of $\bfSigma_N$ does not converge to the subspace spanned by the $K$ component mean vectors, where ``convergence'' is in the sense that the distance   defined in \eqref{eq: dist_subspaces} vanishes as $N\to\infty$. An illustration of this fact is presented in Kannan {\em et al.} \cite[Figure~1]{kannan2005}. However, we can still provide an upper bound for the distance of these two subspaces for more general mixture models. It is proved in Kannan {\em et al.} \cite{kannan2005} that the subspace spanned by the top $K$ singular vectors of $\bfSigma_N$ is close (in terms of sample variances) to the means of samples generated from each component of the mixture (see Lemma \ref{lem: main_kannan} to follow). Define the maximum variance of a set of sample points generated from the $k$-th component of the mixture along any direction in a subspace $\calS$ as
\begin{equation}
 \sigma_{k,\mathcal{S}}^2(\bV) := \max_{\bz \in \mathcal{S}, \|\bz\|_2=1}  \frac{1}{n_k} \sum_{n\in \mathscr{C}_k}\left|\bz^T(\bv_n - \bc_k)\right|^2 ,
\end{equation}
where $\bc_k := \frac{1}{|\mathscr{C}_k|}\sum_{n\in \mathscr{C}_k} \bv_n$ is the centroid of the points in $\mathscr{C}_k$. Recall that from Section~\ref{sec: notations}, we denote $\bar{\bv}:=\frac{1}{N}\sum_{n=1}^N \bv_n$ and $\bZ$ as the centralized matrix of $\bV$. Let $\bar{\bc}_k = \bc_k - \bar{\bv}$. We have $\sigma_{k,\mathcal{S}}^2(\bZ)=\sigma_{k,\mathcal{S}}^2(\bV)$ for any subspace $\calS$ and the the following lemma which is similar to Theorem 1 in Kannan {\em et al.} \cite{kannan2005} holds.  Note that this lemma holds not only for mixtures of log-concave distributions, but also for {\em any}   mixture model.

\begin{lemma}\label{lem: main_kannan}
Let $\mathscr{C}$ be the correct target clustering corresponding to the mixture. 
Let $\mathcal{W}$ be the subspace spanned by the top $K-1$ left singular vectors of $\bZ$.  Then   
 \begin{equation}\label{eq: main_kannan}
  \sum_{k=1}^K n_k \mathrm{d}(\bar{\bc}_k,\mathcal{W})^2 \le (K-1)\sum_{k=1}^K n_k \sigma_{k,\mathcal{W}}^2(\bV),
 \end{equation}
 where $\mathrm{d}(\bar{\bc}_k,\mathcal{W})$ denotes the orthogonal distance of $\bar{\bc}_k$ from subspace $\mathcal{W}$ for any $k\in[K]$.
\end{lemma}

 In addition, recall the notations $\bP_{K-1}, \bP_{-(K-1)}$, and $\bQ_{K-1}$ (cf.~\eqref{eq: basic_svd}) from Section~\ref{sec: proof_main_thm}. Let $\bar{\bu}_k:=\bu_k-\bar{\bu}$ for $k\in[K]$. Since $\sum_{k=1}^K w_k\bar{\bu}_k=0$, $\mathrm{rank}(\bar{\mathbf{\Sigma}}_0)\le K-1$.  It is easy to see that $\mathrm{d}(\bx,\mathcal{W})^2=\|\bx-\bP_{K-1}\bP_{K-1}^T\bx\|_2^2$ for any vector $\bx$ and by denoting $\sigma^2_{k,\max}$ as the maximal eigenvalue of $\bfSigma_k$ (the covariance matrix of the $k$-th component), we obtain the following corollary of Lemma \ref{lem: main_kannan}.

\begin{corollary}\label{coro: close_mean_sampleSpace}
 If we further assume that the mixture is a mixture of log-concave distributions, then for any sufficiently small $\epsilon \in (0,1)$ and all $\eta>0$, if $N\ge C\frac{F^2K^4}{\epsilon^2}\log^5\left(\frac{FK^3}{\epsilon\eta}\right)$, with probability at least $1-\eta$,
 \begin{equation}\label{eq: close_mean_sampleSpace}
  \sum_{k=1}^K w_k \mathrm{d}(\bar{\bu}_k,\mathcal{W})^2 \le\bigg( (K-1)\sum_{k=1}^K w_k \sigma^2_{k,\max} \bigg) + \epsilon.
 \end{equation}
\end{corollary}

The proof of the above corollary is presented in Appendix~\ref{prf: coro: close_mean_sampleSpace}. Because Corollary~\ref{coro: close_mean_sampleSpace} only provides a guarantee for the closeness of centralized component mean vectors to the SVD subspace of centralized samples, we need to further extend it to show that the subspace spanned by all the centralized component mean vectors is close to the SVD subspace of centralized samples. This is described in the following lemma.
\begin{lemma}\label{lem: main_me}
 Let $\bV \in \mathbb{R}^{F \times N}$ be generated from a mixture of $K$ log-concave distributions that satisfies the non-degeneracy condition (cf.\ Definition \ref{def: non-degen}). Using the notations defined above (see, in particular, the definition of $\lambda_{\min}$ in Section \ref{sec: des_main_thm}), 
 \begin{equation}
  \left\|\bP_{K-1}\bP_{K-1}^T - \bQ_{K-1} \bQ_{K-1}^T \right\|_{\mathrm{F}}^2 \le \frac{2\sum_{k=1}^{K} w_k  \mathrm{d}(\bar{\bu}_k,\mathcal{W})^2}{\lambda_{\min}}.
 \end{equation}
\end{lemma}
\begin{proof}
 We have that
 \begin{align}
  \sum_{k=1}^{K} w_k \mathrm{d}(\bar{\bu}_k,\mathcal{W})^2 & = \sum_{k=1}^{K} w_k \|\bar{\bu}_k - \bP_{K-1}\bP_{K-1}^T \bar{\bu}_k\|_2^2 \\
  & = \sum_{k=1}^{K} w_k \bar{\bu}_k^T \bP_{-(K-1)}\bP_{-(K-1)}^T \bar{\bu}_k \\   
  & = \mathrm{tr}(\bA^T\bA),
 \end{align}
where $\bA := \bP_{-(K-1)}^T \bU \bD$, $\bU:=[\bar{\bu}_1,\ldots,\bar{\bu}_K]$,  and $\bD := \mathrm{diag}(\sqrt{w_1},\ldots,\sqrt{w_K})$.
Recall that we write the SVD of $\bar{\mathbf{\Sigma}}_0:=\sum_{k=1}^K w_k\bar{\bu}_k\bar{\bu}_k^T$ as $\bar{\mathbf{\Sigma}}_0=\bQ_{K-1}\bE_{K-1}\bQ_{K-1}^T$. We have 
\begin{align}
 \mathrm{tr}(\bA^T\bA) &= \mathrm{tr}(\bU \bD^2 \bU^T \bP_{-(K-1)}\bP_{-(K-1)}^T) \\
 &= \mathrm{tr}(\bE_{K-1} \bQ_{K-1}^T \bP_{-(K-1)}\bP_{-(K-1)}^T \bQ_{K-1}) \\
 & \ge \lambda_{\min} \mathrm{tr}(\bQ_{K-1}^T \bP_{-(K-1)}\bP_{-(K-1)}^T \bQ_{K-1}),
\end{align}
where the inequality is because the diagonal entries of $\bQ_{K-1}^T \bP_{-(K-1)}\bP_{-(K-1)}^T \bQ_{K-1}$ are nonnegative.

Let $\beta:=\sum_{k=1}^{K} w_k \mathrm{d}(\bar{\bu}_k,\mathcal{W})^2$. By combining the above inequalities, we have 
\begin{equation}
 \|\bP_{-(K-1)}^T \bQ_{K-1}\|_\mathrm{F}^2 \le \frac{\beta}{\lambda_{\min}}.
\end{equation}
Since  
\begin{align}
  \|\bP_{K-1}^T \bQ_{K-1}\|_\mathrm{F}^2 + \|\bP_{-(K-1)}^T \bQ_{K-1}\|_\mathrm{F}^2 &=\|\bP^T \bQ_{K-1}\|_\mathrm{F}^2  = \mathrm{tr}(\bQ_{K-1}^T\bP\bP^T \bQ_{K-1}) \\
& = \mathrm{tr}(\bQ_{K-1}^T\bQ_{K-1})=K-1 = \|\bP_{K-1}^T \bQ\|_\mathrm{F}^2 \\
 & = \|\bP_{K-1}^T \bQ_{K-1}\|_\mathrm{F}^2+\|\bP_{K-1}^T \bQ_{-(K-1)}\|_\mathrm{F}^2,
\end{align}
where $\bQ_{-(K-1)} \in \mathbb{R}^{F\times (F-K+1)}$ is selected such that $[\bQ_{K-1},\bQ_{-(K-1)}]$ is orthogonal, we also have that 
\begin{equation}
 \|\bP_{K-1}^T \bQ_{-(K-1)}\|_\mathrm{F}^2 \le \frac{\beta}{\lambda_{\min}}.
\end{equation}
Now we have 
\begin{align}
\|\bQ_{K-1}\bQ_{K-1}^T - \bP_{K-1} \bP_{K-1}^T\|_{\mathrm{F}}^2 &= \mathrm{tr}(\bQ_{K-1}^T \bP_{-(K-1)}\bP_{-(K-1)}^T \bQ_{K-1})+ \mathrm{tr}(\bP_{K-1}^T \bQ_{-(K-1)}\bQ_{-(K-1)}^T \bP_{K-1})\\
 & = \|\bP_{-(K-1)}^T \bQ_{K-1}\|_\mathrm{F}^2 + \|\bP_{K-1}^T \bQ_{-(K-1)}\|_\mathrm{F}^2 \le \frac{2\beta}{\lambda_{\min}},
\end{align}
concluding the proof of Lemma \ref{lem: main_me}.
\end{proof}

Combining the results of Corollary~\ref{coro: close_mean_sampleSpace} and Lemma~\ref{lem: main_me}, we obtain an upper bound of the distance between the two subspaces, and we can prove a result concerning optimal clusterings of the dimensionality-reduced dataset (via PCA)  generated from a mixture of log-concave distributions. This parallels the procedures for the proof strategy of Theorem~\ref{thm: afterPCA}. 

\subsection{Description of the Theorems for Log-Concave Distributions} \label{sec: des_main_thm_log}
For any $k \in [K]$, let $\sigma^2_{k,\max}$ and $\sigma^2_{k,\min}$ be the maximal and minimal eigenvalues of $\bfSigma_k$ respectively, and define $\bar{\sigma}^2_{\max}:=\sum_{k=1}^K w_k \sigma^2_{k,\max}$ and $\bar{\sigma}^2_{\min}:=\sum_{k=1}^K w_k \sigma^2_{k,\min}$. Other notations are the same as that in previous theorems. Then in a similar manner as  Theorem~\ref{thm: original} for spherical GMMs, we have the following theorem for mixtures of log-concave distributions. 

\begin{theorem} \label{thm: original_log}
Suppose that all the columns of data matrix $\bV \in \mathbb{R}^{F\times N}$ are independently generated from a mixture of $K$ log-concave distributions and $N>F>K$. Assume the mixture model satisfies the non-degeneracy condition. We further assume that 
\begin{equation}\label{eq: main_condition_log}
 0 < \delta_2 := \frac{F\bar{\sigma}_{\max}^2 - (F-K+1)\bar{\sigma}_{\min}^2}{\lambda_{\min} + \bar{\sigma}_{\min}^2 - \bar{\sigma}_{\max}^2} < \zeta(w_{\min}). 
\end{equation}
Then for any sufficiently small $\epsilon\in (0,1)$ and any $t\ge 1$, if $N \ge C\frac{K^2F^4}{\epsilon^2}\log^5\left(\frac{K^2F^3}{\epsilon \eta}\right)$, where $C>0$ depends  on the parameters of the mixture model, we have, with probability at least $1-\eta$,
\begin{align}
 &\mathrm{d}_{\mathrm{ME}}(\mathscr{C},\mathscr{C}^{\mathrm{opt}}) \le \tau\left(\frac{F\bar{\sigma}_{\max}^2 - (F-K+1)\bar{\sigma}_{\min}^2+\epsilon}{\lambda_{\min} + \bar{\sigma}_{\min}^2 - \bar{\sigma}_{\max}^2-\epsilon}\right)(w_{\max}+\epsilon),
\end{align}
where $\mathscr{C}^{\mathrm{opt}}$ is an optimal $K$-clustering for $\bV$.
\end{theorem}
The proof, which is presented in Appendix~\ref{prf: thm: original_log}, is mostly similar to that for Theorem~\ref{thm: original}, except that we employ different concentration inequalities and slightly different bounding strategies (e.g., Lemma \ref{lem: bd_eigen_sum} is required). 

Next, we demonstrate that under similar assumptions on the generating process of the samples (compared to those in Theorem~\ref{thm: original_log}), any optimal clustering for the post-PCA (cf.\ Section~\ref{sec: pca}) dataset is also close to the correct target clustering with high probability.
\begin{theorem}\label{thm: afterPCA_log}
Define $\bar{L}:=\sum_{k=1}^K w_k\left(\|\bu_k\|_2^2+\mathrm{tr}(\mathbf{\Sigma}_k)\right)$.
Let the dataset $\bV\in\mathbb{R}^{F\times N}$ be generated under the same conditions given in Theorem~\ref{thm: original_log} with the separability assumption  in~\eqref{eq: main_condition_log} being modified to 
\begin{equation}\label{eq: main_condition_afterPCA_log}
 0 < \delta_3 := \frac{(K-1)\bar{\sigma}_{\max}^2 + a}{\lambda_{\min}+\bar{\sigma}_{\min}^2 - b} < \zeta(w_{\min}),
\end{equation}
where 
\begin{align}
a&:= (1+K)\bar{L}\sqrt{\frac{2(K-1)\bar{\sigma}_{\max}^2}{\lambda_{\min}}},\quad\mbox{and} \quad b:=(\bar{L}-\|\bar{\bu}\|_2^2)\sqrt{\frac{2(K-1)\bar{\sigma}_{\max}^2}{\lambda_{\min}}} \label{eqn:ab}.
\end{align}
Let $\tilde{\bV} \in \mathbb{R}^{F\times (K-1)}$ be the post-$(K-1)$-PCA dataset of $\bV$. Then for any sufficiently small $\epsilon \in (0,1)$, if $N \ge C \frac{F^2 K^6}{\epsilon^2} \log^5\left(\frac{F^2K^4}{\epsilon\eta}\right)$, where $C>0$ depends on the parameters of the mixture model, we have, with probability at least $1-\eta$,
\begin{equation}
\mathrm{d}_{\mathrm{ME}}(\mathscr{C}, \tilde{\mathscr{C}}^{\mathrm{opt}}) \le \tau\left(\frac{(K-1)\bar{\sigma}_{\max}^2+a+\epsilon}{\lambda_{\min}+\bar{\sigma}_{\min}^2-b-\epsilon}\right)(w_{\max}+\epsilon),
\end{equation}
where $\mathscr{C}$ is the correct target clustering and $\tilde{\mathscr{C}}^{\mathrm{opt}}$ is an optimal $K$-clustering for $\tilde{\bV}$.
\end{theorem}
By using the fact that $\bw=(w_1,w_2,\ldots, w_K )$ is a probability vector and the non-degeneracy condition, we have $\sum_{k=1}^K w_k \|\bu_k\|_2^2 > \|\bar{\bu}\|_2^2$ and thus $b > 0$. The proof of Theorem~\ref{thm: afterPCA_log}, which can be found in Appendix~\ref{prf: thm: afterPCA_log}, is similar to that for Theorem~\ref{thm: afterPCA}, except that the estimate of $\|\bP_{K-1}\bP_{K-1}^T - \bQ_{K-1}\bQ_{K-1}^T\|_\rmF$ is obtained differently.

\subsection{Discussion  and Comparison to Theorems for Spherical GMMs}\label{sec:discuss}
The separability assumption \eqref{eq: main_condition_log} for mixtures of log-concave distributions reduces to \eqref{eq: main_condition} for spherical GMMs because in this case, we have $\bar{\sigma}^2_{\max}=\bar{\sigma}^2_{\min}=\bar{\sigma}^2$. If the mixture model is non-spherical, the separability assumption~\eqref{eq: main_condition_log} is generally  stricter than~\eqref{eq: main_condition}. This is especially the case when $\bar{\sigma}^2_{\max} \gg \bar{\sigma}^2_{\min}$. This implies that non-spherical mixture models are generally more difficult to disambiguate and  learn. For dimensionality-reduced datasets (using PCA), the separability assumption in~\eqref{eq: main_condition_afterPCA_log} is stricter than the separability assumption in~\eqref{eq: main_condition_afterPCA}, even for spherical GMMs, because of the presence of the additional positive terms $a$ and $b$ in~\eqref{eqn:ab}. In addition, unlike that for spherical GMMs, the separability assumption in~\eqref{eq: main_condition_afterPCA_log} for dimensionality-reduced datasets may also  be stricter than the separability assumption in~\eqref{eq: main_condition_log} also because of the same additional positive terms $a$ and $b$.

\section{Other Extensions} \label{sec:other}
In this section, we discuss several interesting and practically-relevant extensions of the preceding results. In Section~\ref{sec: other_dim_reduction}, we show that Theorems \ref{thm: afterPCA} and \ref{thm: afterPCA_log} may be readily extended to other dimensionality-reduction techniques besides PCA/SVD by leveraging results such as~\eqref{eq: distortion_bd}. In Section \ref{sec:alg}, we show that we can apply efficient clustering algorithms to obtain an approximately-optimal clustering which is also close to the correct target clustering.

\subsection{Other Dimensionality-Reduction Techniques} \label{sec: other_dim_reduction}
Our results can be used to prove similar upper bounds for the ME distance between any {\em approximately-optimal} clustering and the correct target clustering. The following corollary follows easily from  Lemma~\ref{lem: ME_dist}.

\begin{corollary}\label{coro: triv_ext}
Consider a $K$-clustering $\mathscr{C}$ with corresponding $\delta$ (cf.\ Lemma~\ref{lem: ME_dist}) and a $K$-clustering $\mathscr{C}'$ that satisfies 
 \begin{equation}
  \mathcal{D}(\bV,\mathscr{C}') \le \gamma \mathcal{D}(\bV,\mathscr{C}^{\mathrm{opt}}), \label{eqn:gammaD}
 \end{equation}
 for some $\gamma \ge 1$. 
 Then if 
 \begin{equation}
  \delta_\gamma := \frac{\gamma \mathcal{D}(\bV,\mathscr{C})-\mathcal{D}^{*}(\bV)}{\lambda_{K-1}(\bS)-\lambda_{K}(\bS)},
 \end{equation}
 satisfies 
\begin{equation}
  \delta_\gamma \le \frac{K-1}{2},\quad\mbox{and}\quad \tau( \delta_\gamma) \le p_{\min}, \label{eqn:new_sep}
 \end{equation} 
we have 
 \begin{equation}
  \mathrm{d}_{\mathrm{ME}}(\mathscr{C}',\mathscr{C}) \le p_{\max} \tau(\delta).
 \end{equation}
\end{corollary}
\begin{proof}
 We have $\delta \le  \delta_\gamma \le \frac{1}{2}(K-1)$ and $\tau(\delta, \delta_\gamma) \le \tau( \delta_\gamma) \le p_{\min}$. Lemma~\ref{lem: ME_dist} thus yields $\mathrm{d}_{\mathrm{ME}}(\mathscr{C},\mathscr{C}') \le p_{\max} \tau(\delta)$. 
\end{proof}

According to the above corollary, we deduce that if we make a stronger separability assumption as in~\eqref{eqn:new_sep}, we can bound the ME distance between any approximately-optimal clustering and the correct target clustering. Therefore, by leveraging~\eqref{eq: distortion_bd}, our theoretical results for dimensionality reduction by PCA (in Theorems~\ref{thm: afterPCA} and~\ref{thm: afterPCA_log}) can be extended to other dimensionality-reduction techniques such as random projection \cite{vempala2005random, bingham2001random, fern2003random, sarlos2006improved} and randomized SVD \cite{drineas1999, sarlos2006improved, halko2011}. We describe these dimensionality-reduction techniques in the following and provide known results for   $\gamma$ satisfying \eqref{eq: distortion_bd}. 

\begin{itemize}
\item 
A {\em random projection} from $F$ dimensions to $D<F$ dimensions is represented by a $D\times F$ matrix, which can be generated as follows \cite{dasgupta2000experiments}: (i) Set each entry of the matrix to be an i.i.d.\ $\calN(0,1)$ random variable; (ii) Orthonormalize the rows of the matrix. Theoretical guarantees for dimensionality-reduction via random projection are usually established by appealing to the well-known Johnson-Lindenstrauss lemma \cite{johnson1984extensions} which says that pairwise distances and inner products are approximately preserved under the random projection. 

\item Because computing an exact SVD is generally expensive, {\em randomized SVD} has gained tremendous interest in solving large-scale problems. For a data matrix $\bV \in \bbR^{F\times N}$, to reduce the dimensionality of the columns from $F$ to $K<F$, one performs a randomized SVD using an $F\times K$ matrix $\bZ_K$. More specifically,  we can adopt the following procedure~\cite{boutsidis2015}: (i) Generate a $D\times F$  ($D>K$) matrix $\bL$ whose entries are i.i.d.\ $\calN(0,1)$ random variables;  (ii) Let $\bA=\bL \bV \in \bbR^{D\times N}$ and orthonormalize the rows of $\bA$ to construct a matrix $\bB$; (iii) Let $\bZ_K \in \bbR^{F\times K}$ be the matrix of top $K$ left singular vectors of $\bV\bB^T \in \bbR^{F\times D}$. Such $\bZ_K$ is expected to satisfy $\|\bV-\bZ_K\bZ_K^T\bV\|_\rmF \approx \|\bV-\bP_K\bP_K^T\bV\|_\rmF$, where $\bP_K$ is the matrix of top $K$ left singular vectors of $\bV$. The key advantage of randomized SVD over exact SVD is that when $D \ll \min\{F,N\}$, the computation of the randomized SVD is significantly faster than the computing of an exact SVD. 

\item We may also employ feature selection techniques such as those described in Boutsidis {\em et al.} \cite{boutsidis2009unsupervised} and Boutsidis and Magdon-Ismail \cite{boutsidis2013deterministic}.
\end{itemize}

\begin{table}
\centering
\caption{Summary of Feature Extraction Techniques}\label{tab: feature_ext}
 \begin{tabular}{|c|c|c|c|} 
\hline
Technique 		& Reference 						& Dimensions 			& $\gamma$  			\\ \hline
PCA/SVD 		&  
    \begin{tabular}{@{}c@{}}Drineas {\em et al.} \cite{drineas2004clustering} \\ Cohen {\em et al.} \cite{cohen2015} \end{tabular}
    & 
    \begin{tabular}{@{}c@{}}$K$ \\ $\lceil K/\epsilon\rceil$ \end{tabular}		& 
    \begin{tabular}{@{}c@{}}2 \\ $1+\epsilon$ \end{tabular}		 	\\ \hline
random projection  	& Cohen {\em et al.} \cite{cohen2015} 					& $O(K/\epsilon^2)$ 		& $1+\epsilon$			\\ \hline
randomized SVD 		& Cohen {\em et al.} \cite{cohen2015} 					& $\lceil K/\epsilon\rceil$		& $1+\epsilon$			\\ \hline
\end{tabular} 
\end{table} 

A subset of the results of Cohen {\em et al.} \cite{cohen2015} is presented in Table~\ref{tab: feature_ext}. This table shows the $\gamma$ such that \eqref{eq: distortion_bd} is satisfied for various reduced dimensions and dimensionality reduction techniques. Even though the results in Table~\ref{tab: feature_ext} appear promising, we observe from numerical experiments that for dimensionality reduction by PCA/SVD, if the data matrix is generated from a spherical GMM, even if it is moderately-well separated (cf.\ Section~\ref{sec: num_exp} to follow), $\mathcal{D}(\bV,\tilde{\mathscr{C}}^{\mathrm{opt}}) \approx \mathcal{D}(\bV,\mathscr{C}^{\mathrm{opt}})$. That is, in this case, $\gamma \approx 1$ even though the reduced dimensionality is $K-1$ or $K$. Furthermore, we show in Theorem~\ref{thm: afterPCA} that for dimensionality reduction by PCA, we require a weaker separability assumption (compared to that in Theorem \ref{thm: original}). However, from Table~\ref{tab: feature_ext}, for SVD, if the reduced dimensionality is $K$, then $\gamma =2$ and we will require a stronger separability assumption according to Corollary~\ref{coro: triv_ext}. 
Therefore, the results for PCA/SVD in Table \ref{tab: feature_ext} are generally pessimistic. This is reasonable because PCA/SVD is  data-dependent and we have assumed specific generative mixture models for our data matrices.

 \subsection{Efficient Clustering Algorithms} \label{sec:alg}
Although $k$-means is a popular heuristic algorithm that attempts to minimize  the sum-of-squares distortion, in general, minimizing this objective is NP-hard and $k$-means only converges to a locally optimal solution. In addition, $k$-means is sensitive to initialization \cite{arthur2007}. Fortunately, there are variants of $k$-means with judiciously chosen initializations that possess  theoretical guarantees, e.g., $k$-means++ \cite{arthur2007}. In addition, efficient variants of $k$-means \cite{ostrovsky2006effectiveness, ackerman2009clusterability} have been proposed to find approximately-optimal clusterings under appropriate conditions. Our theoretical results can be easily combined with these efficient algorithms to produce approximately-optimal clusterings which are also close to the correct target clustering. We demonstrate this by using a result in Ostrovsky {\em et al.}~\cite{ostrovsky2006effectiveness}. Namely, if we denote the optimal distortion with $k\in\bbN$ clusters as $\mathrm{OPT}_k$, Theorem~4.13 in Ostrovsky {\em et al.} \cite{ostrovsky2006effectiveness} states that:
\begin{lemma}\label{lem: ostrovsky2006}
 If $\frac{\mathrm{OPT}_K}{\mathrm{OPT}_{K-1}} \le \epsilon^2$ for a small enough $\epsilon>0$, the randomized algorithm presented before Theorem~4.13 in~\cite{ostrovsky2006effectiveness} returns a solution of cost at most $\left(\frac{1-\epsilon^2}{1-37\epsilon^2} \right)\mathrm{OPT}_K$ with probability $1-O(\sqrt{\epsilon})$ in time $O(FNK+K^3 F)$.
\end{lemma}

We demonstrate that this lemma and the proposed  algorithm    can be combined with our theoretical results to produce further interesting results. For simplicity, we assume the data matrix $\bV \in \bbR^{F\times N}$ is generated from a $K$-component spherical GMM. Then by previous calculations, the lower bound of the distortion for $K-1$ clusters is $\calD_{K-1}^*:=\sum_{k=K-1}^{F}\lambda_{k}(\bS)$ (cf.\ Section~\ref{sec: imp_lemmas}). As $N\to\infty$, $\calD_{K-1}^*$ converges to $\lambda_{\min}+(F-K+2)\bar{\sigma}^2$ in probability. In addition, the distortion for the correct target clustering (with $K$ clusters) converges to $F\bar{\sigma}^2$ in probability. Therefore, if the number of samples is large enough, with high probability,
\begin{equation}
 \frac{\mathrm{OPT}_K}{\mathrm{OPT}_{K-1}} \le \frac{F\bar{\sigma}^2}{\lambda_{\min}+(F-K+2)\bar{\sigma}^2}.
\end{equation}
Thus if $\bar{\sigma}^2$ is sufficiently small or $\lambda_{\min}$ is sufficiently large, by Lemma~\ref{lem: ostrovsky2006}, we can use the algorithm suggested therein to obtain an approximately-optimal clustering for the original dataset. In addition, by Theorem~\ref{thm: original} and Corollary~\ref{coro: triv_ext}, under an appropriate separability assumption, the approximately-optimal clustering is close to the correct target clustering that we ultimately seek.

We have provided one example of an efficient algorithm to obtain an approximately-optimal clustering. Interested readers may refer to the paper by Ackerman and Ben-David \cite{ackerman2009clusterability} which discusses other computationally efficient algorithms with guarantees. 

\section{Conclusion and Future Work}\label{sec:concl}
This paper provides a fundamental information-theoretic understanding about when optimizing the objective function of the $k$-means algorithm returns a clustering that is close to the correct target clustering. To ameliorate computational and memory issues, various dimensionality reduction techniques such as PCA are also considered.

Several natural questions arise from the work herein. 
\begin{enumerate}
\item  Instead of the separability assumptions made herein, we may consider modifying our analyses so that we eventually make less restrictive  {\em pairwise} separability assumptions. This may enable  us to make more direct comparisons between our separability assumptions and similar assumptions in the literature,  such as those in Vempala and Wang \cite{vempala2002} and Kannan {\em et al.} \cite{kannan2005}.
\item Brubaker \cite{brubaker2009robust} considers the {\em robust} learning of mixtures of log-concave distributions. Similarly, we may extend our work to the robust learning of noisy mixtures in which there may be outliers in the data.
\item Besides studying the sum-of-squares distortion measure for $k$-means in~\eqref{eqn:sos}, it may be fruitful to analyze other objective functions such as those for $k$-medians \cite{bradley1997clustering} or min-sum clustering \cite{bartal2001approximating}. These may result in alternative separability assumptions and further insights on the fundamental limits of various clustering tasks.

\item We have provided {\em upper} bounds on the ME distance under certain sufficient (separability) conditions. It would be fruitful to also study {\em necessary} conditions on the separability of the mixture components  to ensure that the ME distance is small. This will possibly result in new separability assumptions which will, in turn, aid in assessing the tightness of our bounds and how they may be improved.
\end{enumerate}

\appendix
\numberwithin{equation}{subsection}

\subsection{Complete Proof of Theorem~\ref{thm: original}} \label{prf: thm: original}
\input{prf_original_doubleColumn}

\subsection{Proof of Corollary~\ref{coro: close_mean_sampleSpace}} \label{prf: coro: close_mean_sampleSpace}
\input{prf_coro_closemean_doubleColumn}

\subsection{Proof of Theorem~\ref{thm: original_log}} \label{prf: thm: original_log}
\input{prf_original_log_doubleColumn}

\subsection{Proof of Theorem~\ref{thm: afterPCA_log}} \label{prf: thm: afterPCA_log}
\input{prf_after_PCA_log_doubleColumn}

\bibliographystyle{IEEEtran}
\bibliography{pcaClustering_ref}

\end{document}

%% file: preamble.tex
\usepackage[mathscr]{eucal}
\usepackage[cmex10]{amsmath}
\usepackage{epsfig,epsf,psfrag}
\usepackage{amssymb,amsmath,amsthm,amsfonts,latexsym}
\usepackage{amsmath,graphicx,bm,xcolor,url}
\usepackage[caption=false]{subfig} 
\usepackage{fixltx2e}
\usepackage{array}
\usepackage{verbatim}
\usepackage{bm}
\usepackage{verbatim}
\usepackage{textcomp}
\usepackage{mathrsfs}
\usepackage{epstopdf}
\usepackage{bbm}

\catcode`~=11 \def\UrlSpecials{\do\~{\kern -.15em\lower .7ex\hbox{~}\kern .04em}} \catcode`~=13 

\allowdisplaybreaks[1]
 
\newcommand{\nn}{\nonumber}

\newcommand{\calD}{\mathcal{D}}
\newcommand{\calE}{\mathcal{E}}
\newcommand{\calF}{\mathcal{F}}

\newcommand{\calN}{\mathcal{N}}

\newcommand{\calS}{\mathcal{S}}

\newcommand{\ba}{\mathbf{a}}
\newcommand{\bA}{\mathbf{A}}

\newcommand{\bB}{\mathbf{B}}
\newcommand{\bc}{\mathbf{c}}

\newcommand{\bD}{\mathbf{D}}

\newcommand{\bE}{\mathbf{E}}

\newcommand{\bH}{\mathbf{H}}

\newcommand{\bI}{\mathbf{I}}

\newcommand{\bL}{\mathbf{L}}

\newcommand{\bM}{\mathbf{M}}

\newcommand{\bN}{\mathbf{N}}

\newcommand{\bp}{\mathbf{p}}
\newcommand{\bP}{\mathbf{P}}

\newcommand{\bQ}{\mathbf{Q}}

\newcommand{\bR}{\mathbf{R}}

\newcommand{\bS}{\mathbf{S}}

\newcommand{\bu}{\mathbf{u}}
\newcommand{\bU}{\mathbf{U}}
\newcommand{\bv}{\mathbf{v}}
\newcommand{\bV}{\mathbf{V}}
\newcommand{\bw}{\mathbf{w}}

\newcommand{\bx}{\mathbf{x}}
\newcommand{\bX}{\mathbf{X}}
\newcommand{\by}{\mathbf{y}}

\newcommand{\bz}{\mathbf{z}}
\newcommand{\bZ}{\mathbf{Z}}

\newcommand{\rmd}{\mathrm{d}}

\newcommand{\rmF}{\mathrm{F}}

\newcommand{\rmp}{\mathrm{p}}

\newcommand{\bbE}{\mathbb{E}}

\newcommand{\bbN}{\mathbb{N}}

\newcommand{\bbP}{\mathbb{P}}

\newcommand{\bbR}{\mathbb{R}}

\newcommand{\scC}{\mathscr{C}}

\DeclareMathAlphabet{\mathbsf}{OT1}{cmss}{bx}{n}
\DeclareMathAlphabet{\mathssf}{OT1}{cmss}{m}{sl}

\DeclareSymbolFont{bsfletters}{OT1}{cmss}{bx}{n}  
\DeclareSymbolFont{ssfletters}{OT1}{cmss}{m}{n}
\DeclareMathSymbol{\bsfGamma}{0}{bsfletters}{'000}
\DeclareMathSymbol{\ssfGamma}{0}{ssfletters}{'000}
\DeclareMathSymbol{\bsfDelta}{0}{bsfletters}{'001}
\DeclareMathSymbol{\ssfDelta}{0}{ssfletters}{'001}
\DeclareMathSymbol{\bsfTheta}{0}{bsfletters}{'002}
\DeclareMathSymbol{\ssfTheta}{0}{ssfletters}{'002}
\DeclareMathSymbol{\bsfLambda}{0}{bsfletters}{'003}
\DeclareMathSymbol{\ssfLambda}{0}{ssfletters}{'003}
\DeclareMathSymbol{\bsfXi}{0}{bsfletters}{'004}
\DeclareMathSymbol{\ssfXi}{0}{ssfletters}{'004}
\DeclareMathSymbol{\bsfPi}{0}{bsfletters}{'005}
\DeclareMathSymbol{\ssfPi}{0}{ssfletters}{'005}
\DeclareMathSymbol{\bsfSigma}{0}{bsfletters}{'006}
\DeclareMathSymbol{\ssfSigma}{0}{ssfletters}{'006}
\DeclareMathSymbol{\bsfUpsilon}{0}{bsfletters}{'007}
\DeclareMathSymbol{\ssfUpsilon}{0}{ssfletters}{'007}
\DeclareMathSymbol{\bsfPhi}{0}{bsfletters}{'010}
\DeclareMathSymbol{\ssfPhi}{0}{ssfletters}{'010}
\DeclareMathSymbol{\bsfPsi}{0}{bsfletters}{'011}
\DeclareMathSymbol{\ssfPsi}{0}{ssfletters}{'011}
\DeclareMathSymbol{\bsfOmega}{0}{bsfletters}{'012}
\DeclareMathSymbol{\ssfOmega}{0}{ssfletters}{'012}

\newcommand{\bSigma	}{\bm{\Sigma}}

\newcommand{\bfP}{\mathbf{P}}

\newcommand{\bfSigma}{\mathbf{\Sigma}}

\newcommand{\scrC}{\mathscr{C}}

\newtheorem{theorem}{Theorem} 
\newtheorem{lemma}{Lemma}

\newtheorem{corollary}{Corollary}
\newtheorem{definition}{Definition}

\newtheorem{remark}{Remark}

\newcommand{\qednew}{\nobreak \ifvmode \relax \else
      \ifdim\lastskip<1.5em \hskip-\lastskip
      \hskip1.5em plus0em minus0.5em \fi \nobreak
      \vrule height0.75em width0.5em depth0.25em\fi}

%% file: prf_original_doubleColumn.tex
\begin{proof}
 To apply Corollary~\ref{coro: main},   we first estimate $\frac{1}{N}\mathcal{D}(\bV,\mathscr{C})$. We have 
 \begin{align}
  \frac{1}{N}\mathcal{D}(\bV,\mathscr{C}) & =\frac{1}{N}\sum_{k=1}^{K} \sum_{n \in \mathscr{C}_k} \|\bv_n - \bc_k\|_2^2  \\
  & = \frac{1}{N}\sum_{k=1}^{K} \left(\sum_{n \in \mathscr{C}_k} \|\bv_n\|_2^2 - n_k \|\bc_k\|_2^2\right) \\
 & = \sum_{k=1}^{K} \frac{n_k}{N} \left(\frac{\sum_{n \in \mathscr{C}_k} \|\bv_n\|_2^2}{n_k}-\|\bc_k\|_2^2\right), \label{eq: distortion_calc}
 \end{align}
 where $n_k:=|\mathscr{C}_k|$.
 By Lemma~\ref{lem:large_dev_Gaussian} and by the property that if $X$ has a sub-Gaussian distribution,\footnote{The {\em sub-Gaussian norm} of a sub-Gaussian random variable $X$ is defined as $\|X \|_{\Psi_2} :=\sup_{p \ge 1} p^{-1/2}\left(\bbE|X|^p\right)^{1/p}$.}  $\|X-\mathbb{E}X\|_{\Psi_2} \le 2\|X\|_{\Psi_2}$ (similarly, if  $X$ has a sub-Exponential distribution, $\|X-\mathbb{E}X\|_{\Psi_1} \le 2\|X\|_{\Psi_1}$) \cite{vershynin2010}, we have for $k\in[K]$, 
 \begin{equation}\label{eq: first_basis0}
  \mathbb{P}\left(\Big|\frac{n_k}{N}-w_k \Big|\ge \frac{w_k}{2}\right) \le e\exp(-C_0 N),
 \end{equation}
 where $C_0 > 0$ is a constant depending on $w_k$, $k\in[K]$. Then with probability at least $1-Ke\exp(-C_0 N)$, we have $\frac{n_k}{N} \ge \frac{w_k}{2}$  for all $k\in[K]$. For brevity, we only consider this case and replace $n_k$ with $N$ in the following inequalities. In addition, for any $0<\epsilon \le \frac{w_{\min}}{2} <1$,  
\begin{equation}\label{eq: first_basis}
 \mathbb{P}\left(\left|\frac{n_k}{N}-w_k \right|\ge \epsilon\right)\le e\exp(-\tau_k N\epsilon^2),
\end{equation}
where $\tau_k>0$ is a constant depending on $w_k$.

By Lemma~\ref{lem:gau_exp}, the square of each entry of a random vector generated from a spherical Gaussian distribution has a sub-Exponential distribution. In addition, for random vector $\bv$ generated from the $k$-th component of the spherical GMM, we have $\mathbb{E} [\bv(f)^2] = \bu_k (f)^2+\sigma_k^2$ for any $f\in[F]$. By Lemma \ref{lem:large_dev}, for any $k \in [K]$,
\begin{align}
  \mathbb{P}\left(\bigg|\frac{1}{n_k}\sum_{n \in \mathscr{C}_k} \|\bv_n\|_2^2 -\left(\|\bu_k\|_2^2+F\sigma_k^2\right)\bigg|\ge \epsilon\right) &\le \sum_{f=1}^{F}\!  \mathbb{P}\left(\bigg|\frac{1}{n_k}{\sum_{n \in \mathscr{C}_k} \bv_n(f)^2}\! -\! \left(\bu_k(f)^2\! +\! \sigma_k^2\right)\bigg|\! \ge\!  \frac{\epsilon}{F}\right)\\
   &\le  2F\exp\left(-\xi_k \frac{N\epsilon^2}{F^2} \right),
\end{align}
where $\xi_k>0$ is a constant depending on $w_k, \sigma_k^2, \bu_k$. 
Similarly, by Lemma~\ref{lem:large_dev_Gaussian}, there exists $\zeta_k>0$ depending on $w_k$, $\sigma_k^2$ and $\bu_k$, such that 
\begin{align}
 \mathbb{P}\left(\left|\|\bc_k\|_2^2-\|\bu_k\|_2^2\right|\ge \epsilon\right)  &\le  \sum_{f=1}^{F} \mathbb{P}\left(\left|\bc_k(f)^2-\bu_k(f)^2\right|\ge \frac{\epsilon}{F}\right)  \\
 &\le  2e F\exp\left(-\zeta_k \frac{N \epsilon^2}{F^2} \right). \label{eq: trick}
\end{align}
The final bound in~\eqref{eq: trick} holds because for any $f$, if $\bu_k(f)=0$, we have  
\begin{align}\label{eq: trick1}
 \mathbb{P}\left(\left|\bc_k(f)^2-\bu_k(f)^2\right|\ge \frac{\epsilon}{F}\right)  &= \mathbb{P}\left(\left|\bc_k(f)\right|\ge \sqrt{\frac{\epsilon}{F}} \right) \\
 &\le e\exp\left(-\zeta_k \frac{N \epsilon}{F} \right) \le e\exp\left(-\zeta_k \frac{N \epsilon^2}{F^2} \right).
\end{align}
On the other hand, if $\bu_k(f)\ne 0$, note that $\left|\bc_k(f)^2-\bu_k(f)^2\right|\ge \frac{\epsilon}{F}$ implies either $\left|\bc_k(f)-\bu_k(f)\right|\ge \frac{\epsilon}{3|\bu_k(f)|F}$ or $\left|\bc_k(f)+\bu_k(f)\right|\ge 3|\bu_k(f)|$. We have
\begin{align}\label{eq: trick2}
\mathbb{P}\left(\left|\bc_k(f)^2-\bu_k(f)^2\right|\ge \frac{\epsilon}{F}\right) &\le \mathbb{P}\left(\left|\bc_k(f)-\bu_k(f)\right|\ge \frac{\epsilon}{3|\bu_k(f)|F}\right) + \mathbb{P}\left(\left|\bc_k(f)+\bu_k(f)\right|\ge 3|\bu_k(f)| \right) \\
&\le 2e\exp\left(-\zeta_k \frac{N \epsilon^2}{F^2} \right).
\end{align}
Now let $d_k:=\frac{1}{n_k}\sum_{n \in \mathscr{C}_k} \|\bv_n\|_2^2-\|\bc_k\|_2^2$ and $L_k:=\|\bu_k\|_2^2+F\sigma_k^2$. Then, there exists a constant $C_1>0$ depending on $\{(w_k,\sigma_k^2,\bu_k)\}_{k\in [K]}$ such that 
\begin{align} 
 &\mathbb{P}\left(\left|\frac{1}{N}\mathcal{D}(\bV,\mathscr{C})-F\bar{\sigma}^2\right| \ge \frac{\epsilon}{2} \right) \le  \sum_{k=1}^{K} \mathbb{P}\left(\left|\frac{n_k}{N} d_k - w_k F \sigma_k^2\right| \ge \frac{\epsilon}{2K} \right)\label{eq: origin_parta}\\
 & \le  \sum_{k=1}^{K} \mathbb{P}\left(\left|\frac{n_k}{N}-w_k\right|F\sigma_k^2 \ge \frac{\epsilon}{4K}\right) + \sum_{k=1}^{K} \mathbb{P}\left(\left|d_k - F\sigma_k^2\right|\frac{n_k}{N} \ge \frac{\epsilon}{4K}\right) \\
 & \le \sum_{k=1}^{K} \mathbb{P}\left(\left|\frac{n_k}{N}-w_k\right| \ge \frac{\epsilon}{4KF\sigma_k^2}\right) + \sum_{k=1}^{K} \mathbb{P } \left(\frac{n_k}{N} \ge 2w_k \right) \nn \\
 & \quad + \sum_{k=1}^{K} \mathbb{P} \left(\bigg|\frac{1}{n_k}{\sum_{n \in \mathscr{C}_k} \|\bv_n\|_2^2}-L_k\bigg|\ge \frac{\epsilon}{16w_k K}\right) + \sum_{k=1}^{K} \mathbb{P}\left(\left|\|\bc_k\|_2^2-\|\bu_k\|_2^2\right|\ge \frac{\epsilon}{16w_k K}\right) \label{eq: origin_partb}\\
 &\le  2K\left((e+2)F+e\right)\exp\left(-C_1 \frac{N \epsilon^2}{F^2 K^2} \right), \label{eq: origin_part1}
\end{align}
where \eqref{eq: origin_parta} is a consequence of~\eqref{eq: distortion_calc}. This proves \eqref{eqn:calD}.

 Now we estimate the positive eigenvalues of $\bS:=\bZ^T \bZ$, where $\bZ$ is the centralized data matrix of $\bV$. Equivalently, we may consider the eigenvalues of $\bar{\mathbf{\Sigma}}_N:=\frac{1}{N}\bZ \bZ^T=\frac{1}{N}\sum_{n=1}^{N} (\bv_n-\bar{\bv})(\bv_n-\bar{\bv})^T$, where $\bar{\bv}=\frac{1}{N}\sum_{n}\bv_n$. The expectation of centralized covariance matrix for the spherical GMM is 
$
 \sum_{k=1}^{K} w_{k} (\bu_{k}-\bar{\bu})(\bu_{k}-\bar{\bu})^{T}+\bar{\sigma}^2\bI=\bar{\mathbf{\Sigma}}.
$
For any $f\in[F]$,
\begin{align}
  \mathbb{P} \left(|\bar{\bv}(f)-\bar{\bu}(f)| \ge \epsilon\right) &\le  \sum_{k=1}^{K} \mathbb{P} \left(\bigg|\frac{n_k}{N} \cdot \frac{\sum_{n\in \mathscr{C}_k}\bv_n(f)}{n_k}-w_k\bu_k(f)\bigg| \ge \frac{\epsilon}{K}\right) \\
 & \le  3Ke \exp\left(-C_2 \frac{N \epsilon^2}{K^2} \right), \label{eq: mean_vec}
\end{align}
and
\begin{align}
\mathbb{P}\left(\|\bar{\bv}\bar{\bv}^T-\bar{\bu}\bar{\bu}^T\|_2 \ge \frac{\epsilon}{2} \right)& \le \mathbb{P}\left(\|(\bar{\bv}-\bar{\bu})(\bar{\bv}-\bar{\bu})^T\|_2 \ge \frac{\epsilon}{6}\right) + \mathbb{P}\left(\|(\bar{\bv}-\bar{\bu})\bar{\bu}^T\|_2 \ge \frac{\epsilon}{6}\right)  + \mathbb{P}\left(\|\bar{\bu}(\bar{\bv}-\bar{\bu})^T\|_2 \ge \frac{\epsilon}{6}\right) \\
 & \le \mathbb{P}\left(\|\bar{\bv}-\bar{\bu}\|_2^2 \ge \frac{\epsilon}{6}\right) + 2\mathbb{P}\left(\|\bar{\bv}-\bar{\bu}\|_2 \ge \frac{\epsilon}{6\|\bar{\bu}\|_2}\right). \label{eqn:three_terms}
\end{align}
Hence, similarly to \eqref{eq: trick}, and by \eqref{eq: mean_vec}, we obtain
\begin{align}
 \mathbb{P}\left(\|\bar{\bv}\bar{\bv}^T-\bar{\bu}\bar{\bu}^T\|_2 \ge \frac{\epsilon}{2} \right) \le  9FKe \exp\left(-C_3 \frac{N \epsilon^2}{F^2 K^2}\right),
\end{align}
where $C_3>0$ is a constant depending on $\{ (w_k,\bu_k ,\sigma_k^2)\}_{ k\in [K]} $. Note that 
\begin{equation}
 \bfSigma_N -\bfSigma  = \sum_{k=1}^K \frac{n_k}{N} \sum_{n \in \scrC_k } \frac{\bv_n \bv_n^T}{n_k} - \sum_{k=1}^K w_k (\bu_k \bu_k^T +\sigma_k^2 \bI).
\end{equation}
Then similar to \eqref{eq: origin_partb} and by Lemma~\ref{lem:cov_est}, we have that for any $t\ge 1$, if 
$N\ge C_4 F^3 K^2 t^2/\epsilon^2$, where $C_4>0$ is a constant depending on $\{ (w_k,\bu_k ,\sigma_k^2)\}_{ k\in [K]} $,  
\begin{align}
   \mathbb{P}\left(\|\bar{\mathbf{\Sigma}}_N-\bar{\mathbf{\Sigma}}\|_2 \ge \frac{\epsilon}{2} \right) &= \mathbb{P}\left( \left\|\left(\mathbf{\Sigma}_N-\mathbf{\Sigma}\right)-\left(\bar{\bv}\bar{\bv}^T-\bar{\bu}\bar{\bu}^T\right) \right\|_2 \ge \frac{\epsilon}{2} \right)\\
 & \le \mathbb{P}\left(\|\mathbf{\Sigma}_N-\mathbf{\Sigma}\|_2 \ge \frac{\epsilon}{4}\right)  + \mathbb{P}\left(\|\bar{\bv}\bar{\bv}^T-\bar{\bu}\bar{\bu}^T\|_2 \ge \frac{\epsilon}{4}\right)\\
 & \le (2K+9FKe)\exp\left(-t^2 F\right). \label{eq: bd_2norm}
\end{align}
 This proves \eqref{eqn:SigmaN}.

Now, if $N\ge C_4 F^5 K^2 t^2/\epsilon^2$, we have
\begin{align}
 \mathbb{P}\left(\left|\frac{1}{N}\mathcal{D}^{*}(\bV)-(F-K+1)\bar{\sigma}^2\right| \ge \frac{\epsilon}{2} \right) &= \mathbb{P}\left(\bigg|\frac{1}{N}\sum_{k=K}^{F} \lambda_k(\bS)-(F-K+1)\bar{\sigma}^2 \bigg| \ge \frac{\epsilon}{2} \right) \\
 & = \mathbb{P}\left( \bigg|\sum_{k=K}^{F} \lambda_k(\bar{\mathbf{\Sigma}}_N)-(F-K+1)\bar{\sigma}^2\bigg| \ge \frac{\epsilon}{2} \right)\\
 & \le \sum_{k=K}^{F} \mathbb{P}\left(|\lambda_k(\bar{\mathbf{\Sigma}}_N)-\lambda_k(\bar{\mathbf{\Sigma}})| \ge \frac{\epsilon}{2(F-K+1)} \right) \\
 & \le (F\! -\! K\! +\! 1)\mathbb{P}\left(\|\bar{\mathbf{\Sigma}}_N\! -\! \bar{\mathbf{\Sigma}}\|_2\!  \ge\!  \frac{\epsilon}{2(F-K+1)} \right)\\
 & \le (F-K+1)(9FKe+2K)\exp(-t^2 F).\label{eq: origin_part2}
\end{align}
 This proves \eqref{eqn:D_star}. 

In addition, if $N\ge C_4 F^3 K^2 t^2/\epsilon^2$, similarly,
\begin{align}
 & \mathbb{P}\left(\left|\frac{1}{N}\lambda_{K-1}(\bS)-\left(\lambda_{K-1}(\bar{\mathbf{\Sigma}}_0)+\bar{\sigma}^2\right)\right| \ge \frac{\epsilon}{2} \right)  \le (9FKe+2K)\exp\left(-t^2 F\right),\label{eq: origin_part4}
 \end{align}
 and
 \begin{align}
  \mathbb{P}\left(\left|\frac{1}{N}\lambda_{K}(\bS)-\bar{\sigma}^2\right| \ge \frac{\epsilon}{2} \right)&\le (9FKe+2K)\exp\left(-t^2 F\right).\label{eq: origin_part3}
\end{align}
 This proves \eqref{eqn:lambda_Kminus1} and \eqref{eqn:lambda_KS}. 
 
Finally, let $p_{\min}=\frac{1}{N}\min_{k}|\mathscr{C}_k|$ and $p_{\max}=\frac{1}{N}\max_{k}|\mathscr{C}_k|$. 
Then if $\epsilon>0$ satisfies~\eqref{eq: epsilon_cond}, when $N \ge C F^5 K^2 t^2/\epsilon^2$ with $C>0$ being a constant depending on $\{ (w_k,\sigma_k^2,\bu_k)\}_{k \in [K]}$,  with probability at least $1-36F^2K\exp\left(-t^2 F\right)$, 
\begin{align} \label{eqn:detailed_eq1}
   \delta  &:=  \frac{\mathcal{D}(\bV, \mathscr{C})-\mathcal{D}^*(\bV)}{\lambda_{K-1}(\bS)-\lambda_{K}(\bS)} \\
 & \le \frac{(F\bar{\sigma}^2 + \frac{\epsilon}{2}) - \left((F-K+1)\bar{\sigma}^2 -\frac{\epsilon}{2}\right)}{(\lambda_{\min}+\bar{\sigma}^2 -\frac{\epsilon}{2}) - (\bar{\sigma}^2 + \frac{\epsilon}{2})}  \\
 &= \frac{(K-1)\bar{\sigma}^2+\epsilon}{\lambda_{\min}-\epsilon}
\end{align}
Note that both $\tau(\mathord{\cdot})$ and $\zeta(\mathord{\cdot})$ are continuous and monotonically increasing functions on $[0,\frac{1}{2}(K-1)]$. Thus by the separability condition~\eqref{eq: epsilon_cond}, we have 
\begin{equation}\label{eq: delta_ineq}
 \delta \le \zeta(w_{\min}-\epsilon) \le \zeta (p_{\min}),
\end{equation}
where the second inequality is from~\eqref{eq: first_basis}.
Note that $\tau(\zeta(p_{\min})) = p_{\min}$. Then~\eqref{eq: delta_ineq} is equivalent to 
\begin{equation}\label{eq: tau_delta}
 \tau(\delta) \le p_{\min}.
\end{equation}
Therefore, by Corollary~\ref{coro: main}, if $N\ge CF^5 K^2 t^2/\epsilon^2$ (where $C>0$ is an appropriate constant depending on $\{  (w_k,\bu_k,\sigma_k^2)\}_{k\in[ K]}$), 
\begin{align}
 \mathrm{d}_{\mathrm{ME}}(\mathscr{C},\mathscr{C}^{\mathrm{opt}}) & \le \tau(\delta) p_{\max} \le \tau(\delta) (w_{\max}+\epsilon) \\
 & \le \tau\left(\frac{(K-1)\bar{\sigma}^2+\epsilon}{\lambda_{\min}-\epsilon}\right) (w_{\max}+\epsilon)\label{eqn:detailed_eqFinal}
\end{align} 
with probability at least $1-36KF^2\exp\left(-t^2 F\right)$. 
\end{proof}

%% file: prf_coro_closemean_doubleColumn.tex
\begin{proof}
Let $\bp_{\bar{\bc},k} = \bP_{-(K-1)}\bP_{-(K-1)}^T\bar{\bc}_k$ and $\bp_{\bar{\bu},k} = \bP_{-(K-1)}\bP_{-(K-1)}^T\bar{\bu}_k$.
Consider,
 \begin{align}
  \mathbb{P}\left(\bigg|\sum_{k=1}^K w_k \mathrm{d}(\bar{\bc}_k,\mathcal{W})^2 -\sum_{k=1}^K w_k \mathrm{d}(\bar{\bu}_k,\mathcal{W})^2 \bigg| \ge \frac{\epsilon}{2}\right) &\le \sum_{k=1}^K \mathbb{P}\left(\left|\|\bp_{\bar{\bc},k}\|_2^2 - \|\bp_{\bar{\bu},k}\|_2^2\right| \ge \frac{\epsilon}{2Kw_k}\right) \\
  & \le \sum_{k=1}^K \mathbb{P}\left(\left|\|\bar{\bc}_k\|_2^2 - \|\bar{\bu}_k\|_2^2\right| \ge \frac{\epsilon}{2Kw_k}\right) \label{eq: proj_less} \\
  & \le 8eK^2F\exp\left(-C_1 \frac{N\epsilon^2}{F^2K^4}\right) \label{eq: close_mean_sampleSpace1},
 \end{align}
where \eqref{eq: proj_less} is because that $\bP_{-(K-1)}\bP_{-(K-1)}^T$ is a projection matrix and $C_1 >0$ is a constant depending on parameters of the mixture model. On the other hand, 
\begin{equation}
 \sigma_{k,\mathcal{W}}^2(\bV) \le \lambda_{1}\bigg(\frac{1}{n_k}\sum_{n\in\scrC_k} (\bv_n-\bc_k)(\bv_n-\bc_k)^T \bigg),
\end{equation}
where $\lambda_{1}(\cdot)$ denotes the largest eigenvalue of a matrix. Consequently, for some fixed $\eta' \in (0,1)$, if $N\ge C_2 \frac{F^2K^4}{\epsilon^2}\log^5\left(\frac{FK^2}{\epsilon \eta'}\right)$  (where $C_2 >0$ is a constant depending on parameters of the mixture model), we have
\begin{align}
 & \mathbb{P}\left((K-1)\sum_{k=1}^K w_k \sigma_{k,\mathcal{W}}^2(\bV) - (K-1)\sum_{k=1}^K w_k \sigma_{k,\max}^2 \ge \frac{\epsilon}{2}\right) \le \sum_{k=1}^K \mathbb{P} \left( \sigma_{k,\mathcal{W}}^2(\bV) - \sigma_{k,\max}^2 \ge \frac{\epsilon}{2K^2 w_k}\right) \\
 & \le  \sum_{k=1}^K\! \mathbb{P} \left( \bigg\|\frac{1}{n_k}\sum_{n\in\scrC_k} (\bv_n\! -\! \bc_k)(\bv_n\! -\! \bc_k)^T \! -\!  \mathbf{\Sigma}_k\bigg\|_2\!  \ge \! \frac{\epsilon}{2K^2 w_k}\right) \label{eq: closemean_smalleq1}\\
 & \le \sum_{k=1}^K \! \mathbb{P} \left( \bigg\|\frac{1}{n_k}\sum_{n\in\scrC_k} (\bv_n\! -\! \bu_k)(\bv_n\! -\! \bu_k)^T \! -\!  \mathbf{\Sigma}_k  \bigg\|_2\!  \ge\!  \frac{\epsilon}{4K^2 w_k}\right) + \sum_{k=1}^K \mathbb{P} \left(\|\bc_k - \bu_k\|_2^2 \ge \frac{\epsilon}{4K^2 w_k}\right)\label{eq: close_mean_sampleSpace_0} \\
&\le 2K\eta' \label{eq: close_mean_sampleSpace2},
\end{align}
where \eqref{eq: closemean_smalleq1} is due to  Lemma~\ref{lem: svd_pert} and \eqref{eq: close_mean_sampleSpace_0} is due to the fact that
 \begin{align}
& \frac{1}{n_k}\sum_{n\in\scrC_k} (\bv_n-\bc_k)(\bv_n-\bc_k)^T= \left[\frac{1}{n_k}\sum_{n\in\scrC_k} (\bv_n-\bu_k)(\bv_n-\bu_k)^T\right] - (\bu_k - \bc_k)  (\bu_k - \bc_k)^T.                                                                                                             
\end{align}
Then with probability at least $1-(8e+2)K\eta'$,
\begin{equation}
 \sum_{k=1}^K w_k \mathrm{d}(\bu_k - \bar{\bu},\mathcal{W})^2 \le\bigg( (K-1)\sum_{k=1}^K w_k \sigma^2_{k,\max} \bigg) + \epsilon.
\end{equation}
Now, we define $\eta := (8e+2)K\eta'$.  Then, we obtain \eqref{eq: close_mean_sampleSpace} with the sample complexity result  as in the statement of Corollary~\ref{coro: close_mean_sampleSpace}.
\end{proof}

%% file: prf_original_log_doubleColumn.tex
\begin{proof}
 We have that the two inequalities concerning $w_k$ in \eqref{eq: first_basis0} and \eqref{eq: first_basis} still hold. In addition, for any $k\in [K]$, if $N \ge C_1\frac{F^2}{\epsilon^2}\log^5\left(\frac{F}{\epsilon \eta}\right)$ with $C_1>0$ being sufficiently large,
 \begin{align}
  \mathbb{P}\left(\bigg|\frac{1}{n_k}{\sum_{n\in\mathcal{I}_k} \|\bv_n\|_2^2} - \left(\|\bu_k\|_2^2+\mathrm{tr}(\mathbf{\Sigma}_k) \right)\bigg| \ge \epsilon\right) &\le \sum_{f=1}^F \mathbb{P}\left(\bigg|\frac{1}{n_k}{\sum_{n\in\mathcal{I}_k} v_n(f)^2} - \left(u_k(f)^2+\mathbf{\Sigma}_k(f,f) \right)\bigg| \ge \frac{\epsilon}{F}\right)   \\
  &\le F\eta.\label{eq: original_log_squaresum}
 \end{align}
By Lemma~\ref{lem:large_dev}, similarly, we have for a sufficiently small $c>0$, 
\begin{align}
 \mathbb{P}\left(\big|\|\bc_k\|_2^2-\|\bu_k\|_2^2\big|\ge \epsilon \right) \le 4F\exp\left(-c \frac{N \epsilon^2}{F^2} \right).
\end{align}
Therefore, by $\sum_{k=1}^K w_k \mathrm{tr}(\mathbf{\Sigma}_k) \le F\bar{\sigma}^2_{\max}$, similar to that for spherical GMMs, if $N \ge C_1 \frac{F^2 K^2}{\epsilon^2}\log^5\left(\frac{F^2K^2}{\epsilon \eta}\right)$, 
\begin{align}
 & \mathbb{P}\left(\frac{1}{N}\mathcal{D}(\bV,\mathscr{I})-F\bar{\sigma}^2_{\max} \ge \frac{\epsilon}{2} \right) \le \eta. \label{eq: original_log_num}
\end{align}
Similarly, we have 
\begin{align}
 & \mathbb{P} \left(|\bar{\bv}(f)-\bar{\bu}(f)| \ge \epsilon\right) \le (4+e)K\exp\left(-c \frac{N \epsilon^2}{K^2} \right).
\end{align}
We also have that 
\begin{align}
 &\mathbb{P}\left(\|\bar{\mathbf{\Sigma}}_N-\bar{\mathbf{\Sigma}}\|_2 \ge \frac{\epsilon}{2} \right)  \le \mathbb{P}\left(\|\mathbf{\Sigma}_N-\mathbf{\Sigma}\|_2 \ge \frac{\epsilon}{4}\right) + \mathbb{P}\left(\|\bar{\bv}\bar{\bv}^T-\bar{\bu}\bar{\bu}^T\|_2 \ge \frac{\epsilon}{4}\right)\\
 & \le \sum_{k=1}^K \bbP \left(\bigg\|\frac{n_k}{N} \frac{\sum_{n\in \scrC_k} \bv_n \bv_n^T}{n_k} - w_k (\bu_k \bu_k^T + \bfSigma_k)\bigg\|_2 \ge \frac{\epsilon}{4K}\right) + \mathbb{P}\left(\|\bar{\bv}\bar{\bv}^T-\bar{\bu}\bar{\bu}^T\|_2 \ge \frac{\epsilon}{4}\right) \\
  & \le \sum_{k=1}^K \bbP \left(\bigg\|\left(\frac{n_k}{N}-w_k\right)(\bu_k \bu_k^T + \bfSigma_k) \bigg\|_2 \ge \frac{\epsilon}{8K}\right) + \sum_{k=1}^K \bbP \left(\frac{n_k}{N} \ge 2w_k\right) + \mathbb{P}\left(\|\bar{\bv}\bar{\bv}^T-\bar{\bu}\bar{\bu}^T\|_2 \!\ge\! \frac{\epsilon}{4}\right)\nonumber\\ 
  &  \quad+ \sum_{k=1}^K \bbP\left(\bigg\|\frac{\sum_{n\in \scrC_k} \bv_n \bv_n^T}{n_k}\!-\!(\bu_k \bu_k^T \!+\! \bfSigma_k)\bigg\|_2 \!\ge\! \frac{\epsilon}{16Kw_k}\right).
\end{align}
Note that the following bound for $\bar{\mathbf{\Sigma}}_N-\bar{\mathbf{\Sigma}}$ is slightly different from~\eqref{eq: bd_2norm} because Lemma~\ref{lem: logconcave_cov} requires that the log-concave distribution has {\em zero mean}. 
Furthermore, recall that $\bc_k := \frac{1}{|\scrC_k|}\sum_{n\in \scrC_k}\bv_n$, we have 
\begin{align}
 & \bbP\left(\bigg\|\frac{1}{n_k}{\sum_{n\in \scrC_k} \bv_n \bv_n^T}-(\bu_k \bu_k^T + \bfSigma_k)\bigg\|_2 \ge \frac{\epsilon}{16Kw_k}\right)\nn\\
 &\le \bbP\left(\bigg\|\frac{1}{n_k}{\sum_{n\in \scrC_k} (\bv_n-\bu_k) (\bv_n-\bu_k)^T}-\bfSigma_k \bigg\|_2 \ge \frac{\epsilon}{32Kw_k}\right) + \bbP \left(\|\bc_k \bu_k^T + \bu_k \bc_k^T - 2 \bu_k \bu_k^T\|_2 \ge \frac{\epsilon}{32Kw_k}\right) \label{eqn:ck} \\
 & \le \bbP\left(\bigg\|\frac{1}{n_k}{\sum_{n\in \scrC_k} (\bv_n-\bu_k) (\bv_n-\bu_k)^T}-\bfSigma_k \bigg\|_2 \ge \frac{\epsilon}{32Kw_k}\right) + 2 \bbP \left(\|\bc_k -\bu_k\|_2 \ge \frac{\epsilon}{64Kw_k\|\bu_k\|_2}\right).
\end{align}
Therefore, we have that if $N\ge C_1\frac{F^2 K^2}{\epsilon^2}\log^5\left(\frac{FK^2}{\epsilon \eta}\right)$, with probability at least $1-\eta$,
\begin{equation}
 \|\bar{\mathbf{\Sigma}}_N-\bar{\mathbf{\Sigma}}\|_2 < \frac{\epsilon}{2}.
\end{equation}
Now, note that by Lemma~\ref{lem: bd_eigen_sum}, for $k \le F <N$, 
\begin{align}
\lambda_k\left(\bar{\mathbf{\Sigma}}\right)&=\lambda_k\left(\bar{\mathbf{\Sigma}}_0 + \sum_{k=1}^K w_k\mathbf{\Sigma}_k\right) \\
& \ge \lambda_k(\bar{\mathbf{\Sigma}}_0) + \lambda_F\left(\sum_{k=1}^K w_k\mathbf{\Sigma}_k \right) \\
& \ge \lambda_k(\bar{\mathbf{\Sigma}}_0) + \sum_{k=1}^K w_k\lambda_F(\mathbf{\Sigma}_k) \\
& = \lambda_k(\bar{\mathbf{\Sigma}}_0) + \bar{\sigma}^2_{\min}.
\end{align}
 Therefore, if $N\ge C_1\frac{F^4 K^2}{\epsilon^2}\log^5\left(\frac{F^2K^2}{\epsilon \eta}\right)$, we have
\begin{align}
 \mathbb{P}\left(\frac{1}{N}\mathcal{D}^{*}(\bV)-(F-K+1)\bar{\sigma}^2_{\min} \le -\frac{\epsilon}{2} \right) &= \mathbb{P}\left(\frac{1}{N}\sum_{k=K}^{F} \lambda_k(\bS)-(F-K+1)\bar{\sigma}^2_{\min} \le -\frac{\epsilon}{2} \right) \\
 & \le \sum_{k=K}^{F} \mathbb{P}\left(|\lambda_k(\bar{\mathbf{\Sigma}}_N)-\lambda_k(\bar{\mathbf{\Sigma}})| \ge \frac{\epsilon}{2(F-K+1)} \right) \\
 & \le 2(F-K+1)\eta.
\end{align}
Or more concisely, if $N\ge C_1\frac{F^4 K^2}{\epsilon^2}\log^5\left(\frac{F^3K^2}{\epsilon \eta}\right)$,
\begin{equation}
 \mathbb{P}\left(\frac{1}{N}\mathcal{D}^{*}(\bV)-(F-K+1)\bar{\sigma}^2_{\min} \le -\frac{\epsilon}{2} \right) \le \eta.
\end{equation}
Similarly, by the inequalities $\lambda_{K-1}(\bar{\mathbf{\Sigma}}) \ge \lambda_{K-1}(\bar{\mathbf{\Sigma}}_0)+\bar{\sigma}^2_{\min}$ and $\lambda_{K}(\bar{\mathbf{\Sigma}}) \le \bar{\sigma}^2_{\max}$,   if $N\ge C_1\frac{F^2 K^2}{\epsilon^2}\log^5\left(\frac{FK^2}{\epsilon \eta}\right)$,
\begin{align}
 \mathbb{P}\left(\frac{1}{N}\lambda_{K-1}(\bS)-\left(\lambda_{\min} +\bar{\sigma}^2_{\min}\right) \le -\frac{\epsilon}{2} \right)  & \le \eta,\\
  \mathbb{P}\left(\frac{1}{N}\lambda_{K}(\bS)-\bar{\sigma}^2_{\max} \ge \frac{\epsilon}{2} \right)&\le \eta.
\end{align}
Combining these results with Corollary~\ref{coro: main}, we similarly obtain the conclusion we desire.
\end{proof}

%% file: prf_after_PCA_log_doubleColumn.tex
\begin{proof}
 We use the same notations as those in the proof of Theorem~\ref{thm: afterPCA} and in the statement of Theorem~\ref{thm: afterPCA_log}.  Since $\bQ_{K-1} \in \mathbb{R}^{F\times (K-1)}$ has orthonormal columns, 
\begin{equation}
 \mathrm{tr}(\bQ_{K-1}^T\bfSigma_k\bQ_{K-1}) \le \sum_{j=1}^{K-1} \lambda_j(\bfSigma_k) \le (K-1)\sigma^2_{k,\max}
\end{equation}
for all $k \in [K]$. 
Therefore, similar to that for the case for the original dataset (cf.~the inequality in~\eqref{eq: original_log_num}), if $N\ge C_1 \frac{K^4}{\epsilon^2}\log^5\left(\frac{K^4}{\epsilon\eta}\right)$ with $C_1>0$ being  sufficiently large,
 \begin{equation}
  \mathbb{P}\left(\frac{1}{N}\mathcal{D}(\hat{\bV},\mathcal{I})-(K-1)\bar{\sigma}_{\max}^2 \ge \frac{\epsilon}{2}\right) \le \eta.
 \end{equation}
In addition, we have $\lambda_{K-1}(\bQ_{K-1}^T \bfSigma_k \bQ_{K-1})  \ge \lambda_F(\bfSigma_k)=\sigma_{k,\min}^2$.\footnote{Indeed, assume, to the contrary, that $\lambda_{K-1}(\bQ_{K-1}^T \bfSigma_k \bQ_{K-1})  < \sigma_{k,\min}^2$. Then there is a $\lambda < \sigma_{k,\min}^2$ and a corresponding unit vector $\bx \in \mathbb{R}^{K-1}$, such that $\bQ_{K-1}^T \bfSigma_k \bQ_{K-1} \bx = \lambda \bx$. Thus, $\sigma_{k,\min}^2 \|\bx\|_2^2=\sigma_{k,\min}^2 \|\bQ_{K-1} \bx\|_2^2 \le \bx^T\bQ_{K-1}^T \bfSigma_k \bQ_{K-1} \bx = \lambda \|\bx\|_2^2 < \sigma_{k,\min}^2 \|\bx\|_2^2$, which is a contradiction.} Similarly, we have $\lambda_1(\bQ_{K-1}^T \bfSigma_k \bQ_{K-1})  \le \lambda_1(\bfSigma_k)=\sigma_{k,\max}^2$. Thus if $N\ge C_1 \frac{K^4}{\epsilon^2}\log^5\left(\frac{K^3}{\epsilon\eta}\right)$,
\begin{align}
 \mathbb{P}\left(\frac{1}{N}\lambda_{K-1}(\hat{\bS}) - (\lambda_{\min}+\bar{\sigma}_{\min}^2) \le -\frac{\epsilon}{2}\right) &\le \eta. 
\end{align}
 Recall that we write $\bR:=\bQ_{K-1}\bQ_{K-1}^T-\bP_{K-1}\bP_{K-1}^T$. Let $r:=\sqrt{\frac{2(K-1)\bar{\sigma}_{\max}^2}{\lambda_{\min}}}$. Combining Corollary~\ref{coro: close_mean_sampleSpace} and Lemma~\ref{lem: main_me}, we have that if $N\ge C\frac{F^2K^4}{\epsilon^2}\log^5\left(\frac{FK^3}{\epsilon\eta}\right)$, with probability at least $1-\eta$,
\begin{equation}
 \|\bR\|_\rmF \le r + C_2 \epsilon, \label{eq: afterPCA_log_main}
\end{equation}
where $C_2>0$ is sufficiently large.
In addition, using the inequalities 
\begin{align}
 \frac{1}{N}\big|\mathcal{D}(\hat{\bV},\mathscr{I})-\mathcal{D}(\tilde{\bV},\mathscr{I})\big| & \le \frac{ 1+K}{N}\|\bV\|_\mathrm{F}^2 \|\bR\|_\mathrm{F} \\
 \frac{1}{N}\left\|\hat{\bS}-\tilde{\bS}\right\|_2 & \le \|\bR\|_\mathrm{F} \|\bZ\|_{\rmF}^2, 
\end{align}
we deduce that 
\begin{align}
 &\mathbb{P}\left(\frac{1}{N}\big|\mathcal{D}(\hat{\bV},\mathscr{I})-\mathcal{D}(\tilde{\bV},\mathscr{I})\big| - a \ge \frac{\epsilon}{2}\right) \le \bbP \left(\frac{1}{N}\|\bV\|_\mathrm{F}^2 \|\bR\|_\rmF - \bar{L} r \ge \frac{\epsilon}{2(1+K)}\right) \label{eqn:split_3}  \\
 & \le \mathbb{P}\left(\Big(\frac{1}{N}\|\bV\|_\mathrm{F}^2-\bar{L}\Big)r \ge \frac{\epsilon}{4(1+K)}\right)+ \mathbb{P}\left(\frac{1}{N}\|\bV\|_\mathrm{F}^2 \ge \bar{L}+1\right) + \mathbb{P}\left(\|\bR\|_\mathrm{F}-r \ge \frac{\epsilon}{4(1+K)(\bar{L}+1)}\right). \label{eqn:split3} 
\end{align}
Therefore, by \eqref{eq: original_log_squaresum} and \eqref{eq: afterPCA_log_main}, we obtain that if $N \ge C_1 \frac{F^2 K^6}{\epsilon^2} \log^5\left(\frac{F^2K^4}{\epsilon\eta}\right)$,
\begin{equation}
 \mathbb{P}\left(\frac{1}{N}\big|\mathcal{D}(\hat{\bV},\mathscr{I})-\mathcal{D}(\tilde{\bV},\mathscr{I})\big| - a \ge \frac{\epsilon}{2}\right) \le \eta.
\end{equation}
Similarly, when $N \ge C_1 \frac{F^2 K^6}{\epsilon^2} \log^5\left(\frac{F^2K^4}{\epsilon\eta}\right)$,
\begin{equation}
 \mathbb{P}\left(\frac{1}{N}\left\|\hat{\bS}-\tilde{\bS}\right\|_2 -b  \ge \frac{\epsilon}{2}\right) \le \eta.
\end{equation}
Combining these results with Corollary~\ref{coro: main}, we obtain the desired conclusion.
\end{proof}